\documentclass[sigconf,screen]{acmart}

\usepackage{multirow}
\usepackage{subcaption}
\usepackage{xspace}

\newtheorem{theorem}{Theorem}
\newtheorem{lemma}{Lemma}
\newtheorem{remark}{Remark}
\newcommand{\tabincell}[2]{\begin{tabular}{@{}#1@{}}#2\end{tabular}} 

\newcommand{\alg}{\textsc{PathCon}\xspace}
\newcommand{\xhdr}[1]{\vspace{2mm}{\noindent\bfseries #1}.}
\newcommand{\rl}[1]{{\textsf{#1}}\xspace}

\copyrightyear{2021}
\acmYear{2021}
\setcopyright{none}
\acmConference[KDD '21]{Proceedings of the 27th ACM SIGKDD Conference on Knowledge Discovery and Data Mining}{August 14--18, 2021}{Virtual Event, Singapore}
\acmBooktitle{Proceedings of the 27th ACM SIGKDD Conference on Knowledge Discovery and Data Mining (KDD '21), August 14--18, 2021, Virtual Event, Singapore}
\acmPrice{15.00}
\acmDOI{10.1145/3447548.3467247}
\acmISBN{978-1-4503-8332-5/21/08}

\begin{document}
\fancyhead{}

\title{Relational Message Passing for Knowledge Graph Completion}

\author{Hongwei Wang}
\affiliation{
	\institution{Stanford University}
	\country{Stanford, California, United States}}
\email{hongweiw@cs.stanford.edu}

\author{Hongyu Ren}
\affiliation{
	\institution{Stanford University}
	\country{Stanford, California, United States}}
\email{hyren@cs.stanford.edu}

\author{Jure Leskovec}
\affiliation{
	\institution{Stanford University}
	\country{Stanford, California, United States}}
\email{jure@cs.stanford.edu}

\begin{abstract}
	Knowledge graph completion aims to predict missing relations between entities in a knowledge graph.
	In this work, we propose a \textit{relational message passing} method for knowledge graph completion.
	Different from existing embedding-based methods, relational message passing only considers edge features (i.e., relation types) without entity IDs in the knowledge graph, and passes relational messages among edges iteratively to aggregate neighborhood information.
	Specifically, two kinds of neighborhood topology are modeled for a given entity pair under the relational message passing framework:
	(1) \textit{Relational context}, which captures the relation types of edges adjacent to the given entity pair;
	(2) \textit{Relational paths}, which characterize the relative position between the given two entities in the knowledge graph.
	The two message passing modules are combined together for relation prediction.
	Experimental results on knowledge graph benchmarks as well as our newly proposed dataset show that, our method \alg outperforms state-of-the-art knowledge graph completion methods by a large margin.
	\alg is also shown applicable to inductive settings where entities are not seen in training stage, and it is able to provide interpretable explanations for the predicted results.
	The code and all datasets are available at \url{https://github.com/hwwang55/PathCon}.
\end{abstract}



\maketitle

\section{Introduction}
	Knowledge graphs (KGs) store structured information of real-world entities and facts.
	A KG usually consists of a collection of triplets. Each triplet $(h, r, t)$ indicates that head entity $h$ is related to tail entity $t$ through relationship type $r$.
	Nonetheless, KGs are often incomplete and noisy.
	To address this issue, researchers have proposed a number of KG completion methods to predict missing links/relations in KGs \cite{bordes2013translating, trouillon2016complex, yang2015embedding, sun2019rotate, kazemi2018simple, zhang2019quaternion, galarraga2015fast, yang2017differentiable, ho2018rule, zhang2019iteratively, sadeghian2019drum}.
	
	\begin{figure}[!t]
		\centering
		\begin{subfigure}[b]{0.47\textwidth}
   			\includegraphics[width=\textwidth]{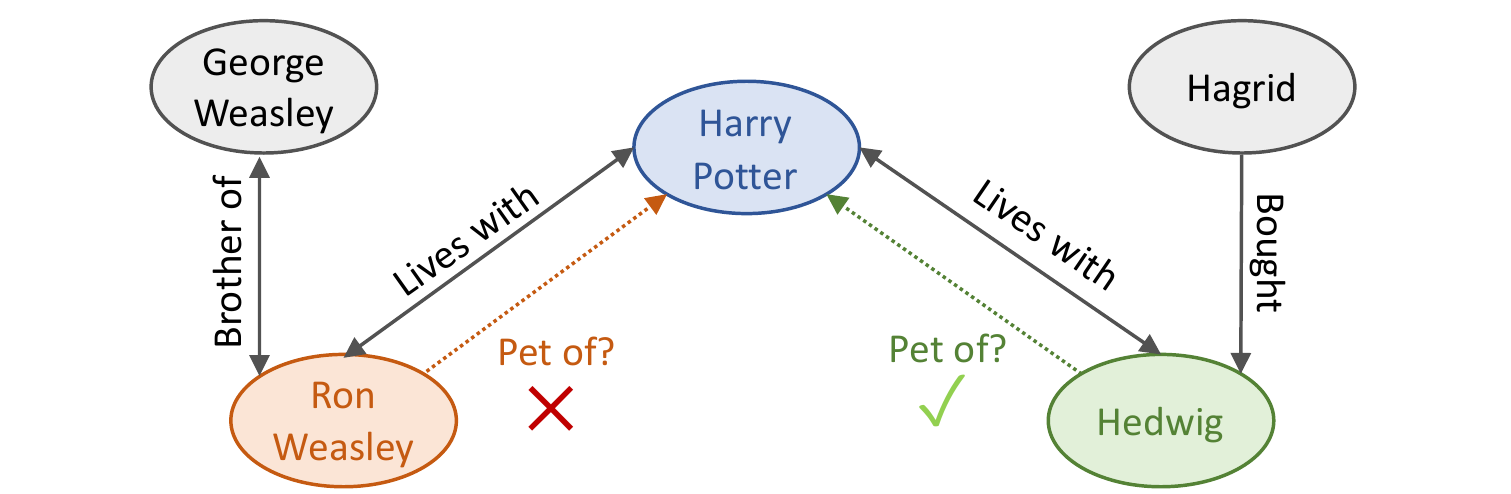}
   	    	\caption{Consider we aim to predict whether \rl{Ron Weasley} or \rl{Hedwig} is a \rl{Pet of} \rl{Harry Potter}. Both entities have the same relational path (\rl{Lives with}) to \rl{Harry Potter} but they have distinct relational context: \rl{Ron Weasley} has $\{\rl{Brother of}, \rl{Lives with}\}$, while \rl{Hedwig} has $\{\rl{Bought}, \rl{Lives with}\}$. Capturing the relational context of entities allows our model to make a distinction between \rl{Ron Weasley}, who is a person, and \rl{Hedwig}, which is an owl.}
   			\label{fig:kg_1}
		\end{subfigure}
		\begin{subfigure}[b]{0.47\textwidth}
	    	\includegraphics[width=\textwidth]{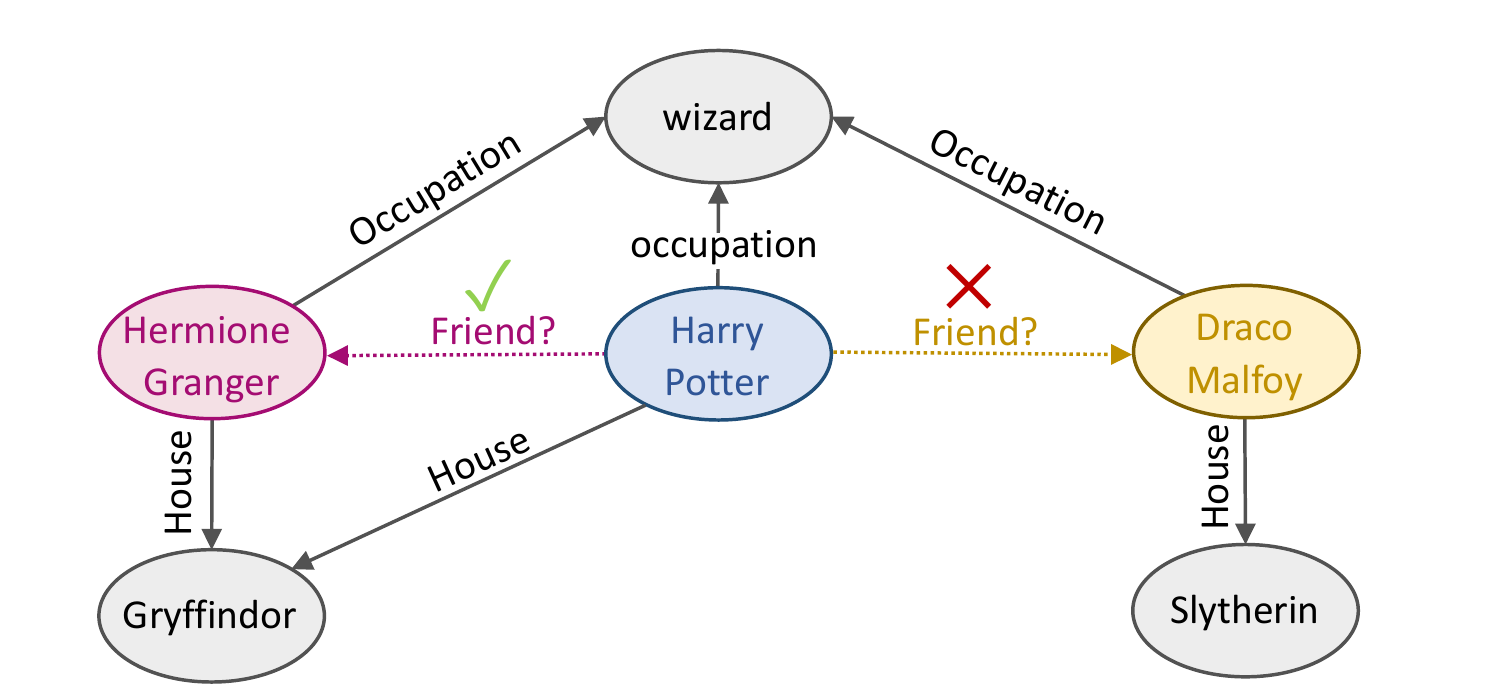}
			\caption{Two head entities \rl{Hermione Granger} and \rl{Draco Malfoy} have the same relational context $\{\rl{Occupation}, \rl{House}\}$, but different relational paths to the tail entity \rl{Harry Potter} \{(\rl{House}, \rl{House}), (\rl{Occupation}, \rl{Occupation})\} vs. \{(\rl{Occupation}, \rl{Occupation})\}, which allows our model to predict friendship between \rl{Harry Potter} and \rl{Hermione Granger} vs. \rl{Draco Malfoy}.}
			\label{fig:kg_2}
		\end{subfigure}
		\caption{(a) Relational context of an entity and (b) relational paths between entities. Our model is able to capture both.}
		\label{fig:intro}
	\end{figure}
	
	In general, relation types are not uniformly distributed over a KG but spatially correlated with each other.
	For example, the neighboring relations of ``\rl{graduated from}'' in the KG are more likely to be ``\rl{person.birthplace}'' and ``\rl{university.location}'' rather than ``\rl{movie.language}''.
	Therefore, for a given entity pair $(h, t)$, characterizing the relation types of neighboring links of $h$ and $t$ will provide valuable information when inferring the relation type between $h$ and $t$.
	Inspired by recent success of graph neural networks \cite{kipf2017semi, hamilton2017inductive, xu2019powerful}, we propose using \textit{message passing} to capture the neighborhood structure for a given entity pair.
	However, traditional message passing methods usually assume that messages are associated with nodes and messages are passed from nodes to nodes iteratively, which are not suitable for KGs where edge features (relation types) are more important.
	
	\xhdr{Relational message passing}
	To address the above limitation, we propose \textit{relational message passing} for KG completion.
	Unlike traditional node-based message passing, relational message passing \textit{only} considers edge features (relation types), and passes messages of an edge directly to its neighboring edges.
	Note that since relational message passing only models relations rather than entities, it brings three additional benefits compared with existing knowledge graph embedding methods \cite{bordes2013translating, trouillon2016complex, yang2015embedding, sun2019rotate, kazemi2018simple, zhang2019quaternion}:
	(1) it is \textit{inductive}, since it can handle entities that do not appear in the training data during inference stage;
	(2) it is \textit{storage-efficient}, since it does not calculate embeddings of entities; and
	(3) it is \textit{explainable}, since it is able to provide explainability for predicted results by modeling the correlation strength among relation types.
	However, a potential issue of relational message passing is that its computational complexity is significantly higher than node-based message passing (Theorem \ref{thm:2}).
	To solve this issue, we propose \textit{alternate relational message passing} that passes relational messages between nodes and edges \textit{alternately} over the KG.
	We prove that alternate message passing scheme greatly improves time efficiency and achieves \textit{the same order of computational complexity} as traditional node-based message passing (Theorem \ref{thm:1} and \ref{thm:3}).
	
	\xhdr{Relational context and relational paths}
	Under the alternate relational message passing framework, we explore two kinds of local subgraph topology for a given entity pair $(h, t)$ (see Figure \ref{fig:intro} for an illustrating example):
	(1) \textit{Relational context}.
	It is important to capture the neighboring relations of a given entity in the KG, because neighboring relations provide us with valuable information about what is the nature or the ``type'' of the given entity (Figure \ref{fig:kg_1}).
	Many entities in KGs are not typed or are very loosely typed, so being able to learn about the entity and its context in the KG is valuable.
	We design a multi-layer relational message passing scheme to aggregate information from multi-hop neighboring edges of $(h, t)$.
	(2) \textit{Relational paths}.
	Note that modeling only relational context is not able to identify the relative position of $(h, t)$.
	It is also important to capture the set of relational paths between $(h, t)$ (Figure \ref{fig:kg_2}).
	Here different paths of connections between the entities reveal the nature of their relationship and help with the prediction.
	Therefore, we calculate all relational paths connecting $h$ and $t$ in the KG and pass relational messages along these paths.
	Finally, we use an attention mechanism to selectively aggregate representations of different relational paths, then combine the above two modules together for relation prediction.
	
	\xhdr{Experiments}
    We conduct extensive experiments on five well-known KGs as well as a new KG proposed by us, \textit{DDB14 dataset}.
    Experimental results demonstrate that our proposed model \alg (short for relational PATHs and CONtext) significantly outperforms state-of-the-art KG completion methods, for example, the absolute Hit@1 gain over the best baseline is $16.7\%$ and $6.3\%$ on WN18RR and NELL995, respectively.
    Our ablation studies show the effectiveness of our approach and demonstrate the importance of relational context as well as relational paths.
    Our method is also shown to maintain strong performance in inductive KG completion, and it provides high explainability by identifying important relational context and relation paths for a given predicted relation.
    
    \xhdr{Contributions}
    Our key contributions are listed as follows:
    \begin{itemize}
    	\item We propose \textit{alternate relational message passing} framework for KG completion, which is \textit{inductive}, \textit{storage-efficient}, \textit{explainable}, and \textit{computationally efficient} compared with existing embedding-based methods;
    	\item Under the proposed framework, we explore two kinds of subgraph topology: \textit{relational context} and \textit{relational paths}, and show that they are critical to relation prediction;
    	\item We propose a new KG dataset DDB14 (Disease Database with 14 relation types) that is suitable for KG-related research.
    \end{itemize}

    \begin{figure}[t]
	    \centering
  	    \includegraphics[width=0.42\textwidth]{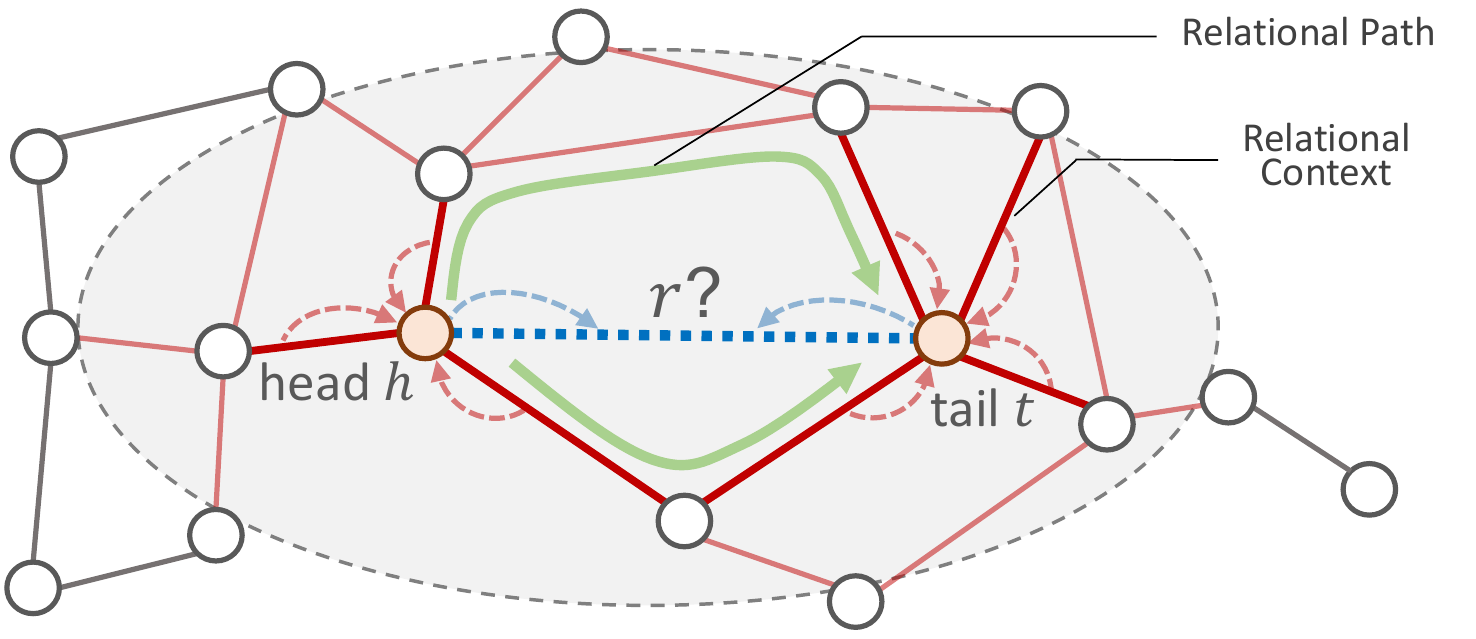}
  	    \caption{An example of \alg considering both the relational context within 2 hops of the head and the tail entities (denoted by red edges) and relational paths of length up to 3 relations that connect head to tail (denoted by green arrows). Context and paths are captured based on relation types (not entities) they contain. By combining the context and paths \alg predicts the probability of relation $r$.}
  	    \label{fig:model}
    \end{figure}

\section{Problem Formulation}
\label{sec:problem_formulation}
	Let $\mathcal G = (\mathcal V, \mathcal E)$ be an instance of a knowledge graph, where $\mathcal V$ is the set of nodes and $\mathcal E$ is the set of edges.
	Each edge $e$ has a relation type $r \in \mathcal R$.
	Our goal is to predict missing relations in $\mathcal G$, i.e., given an entity pair $(h, t)$, we aim to predict the relation of the edge between them.\footnote{Some of the related work formulates this problem as predicting the missing tail (head) entity given a head (tail) entity and a relation. The two problems are actually reducible to each other: Given a model $\Phi (\cdot | h, t)$ that outputs the distribution over relation types for an entity pair $(h, t)$, we can then build a model $\Gamma (\cdot | h, r) = \textsc{SoftMax}_t \left( \Phi (r | h, t) \right)$ that outputs the distribution over tail entities given $h$ and $r$, and vice versa. Since the two problems are equivalent, we only focus on relation prediction in this work.}
	Specifically, we aim to model the distribution over relation types given a pair of entities $(h, t)$: $p (r | h, t)$.
	This is equivalent to modeling the following term
	\begin{equation}
	    \label{eq:bayes}
		p (r | h, t) \propto p (h, t | r) \cdot p (r)
	\end{equation}
	according to Bayes' theorem.
	In Eq. (\ref{eq:bayes}), $p (r)$ is the prior distribution over relation types and serves as the regularization of the model.
	Then the first term can be further decomposed to
	\begin{equation}
	    \label{eq:bayes2}
		p (h, t | r) = \frac{1}{2} \Big( p (h | r) \cdot p (t | h, r) + p (t | r) \cdot p (h | t, r) \Big).
	\end{equation}
		
	Eq. (\ref{eq:bayes2}) sets up the guideline for designing our model.
	The term $p (h | r)$ or $p (t | r)$ measures the likelihood of an entity given a particular relation.
	Since our model does not consider the identity of entities, we use an entity's \textit{local relational subgraph} instead to represent the entity itself, i.e., $p \big( C(h) | r \big)$ and $p \big( C(t) | r \big)$ where $C(\cdot)$ denotes the local relational subgraph of an entity.
	This is also known as \textit{relational context} for $h$ and $t$.
    The term $p (t | h, r)$ or $p (h | t, r)$ in Eq. (\ref{eq:bayes2}) measures the likelihood of how $t$ can be reached from $h$ or the other way around given that there is a relation $r$ between them.
	This inspires us to model the \textit{relational paths} between $h$ and $t$ in the KG.
	In the following we show how to model the two factors in our method and how they contribute to relation prediction.
		
	\begin{table}[t]
		\centering
		\small
		\setlength{\tabcolsep}{10pt}
		\begin{tabular}{c|c}
			\hline
		    Symbol & Description \\
        	\hline
    		$h, t$ & Head entity and tail entity \\
    		$r$ & Relation type \\
        	$s_e^i$ & Hidden state of edge $e$ at iteration $i$ \\
    		$m_v^i$ & Message of node $v$ at iteration $i$ \\
        	$\mathcal N(e)$ & Endpoint nodes of edge $e$ \\
        	$\mathcal N(v)$ & Neighbor edges of node $v$ \\
        	$s_{(h,t)}$ & Context representation of the entity pair ($h$,$t$) \\
        	$s_{h\to t}$ & Path representation of all paths from $h$ to $t$ \\
        	$\alpha_P$ & Attention weight of path $P$ \\
        	$\mathcal P_{h\rightarrow t}$ & Set of paths from $h$ to $t$ \\
        	\hline
		\end{tabular}
		\vspace{0.05in}
		\caption{Notation used in this paper.}
		\vspace{-0.2in}
		\label{table:symbols}
	\end{table}

\section{Our Approach}
	In this section, we first introduce the relational message passing framework, then present two modules of the proposed \alg: relational context message passing and relational path message passing.
	Notations used in this paper are listed in Table \ref{table:symbols}.
	
	\subsection{Relational Message Passing Framework}
	\label{sec:framework}
		\textbf{Traditional node-based message passing}.
		We first briefly review traditional node-based message passing method for general graphs.
		Assume that each node $v$ is with feature $x_v$.
		Then the message passing runs for multiple timesteps over the graph, during which the hidden state $s^i_v$ of each node $v$ in iteration $i$ is updated by
		\begin{align}
			m^i_v =& A \left( \big \{ s^i_u \big \}_{u \in \mathcal N (v)} \right), \label{eq:node_based_mp_1} \\
			s^{i+1}_v =& U \left( s^i_v, m^i_v \right), \label{eq:node_based_mp_2}
		\end{align}
		where $m^i_v$ is the message received by node $v$ in iteration $i$, $\mathcal N(v)$ denotes the set of neighbor nodes of $v$ in the graph, $A(\cdot)$ is message aggregation function, and $U(\cdot)$ is node update function.
		The initial hidden state $s^0_v = x_v$.
		
		The above framework, though popular for general graphs and has derived many variants such as GCN \cite{kipf2017semi}, GraphSAGE \cite{hamilton2017inductive}, and GIN \cite{xu2019powerful}, faces the following challenges when applied to knowledge graphs:
		(1) Unlike general graphs, in most KGs, edges have features (relation types) but nodes don't, which makes node-based message passing less natural for KGs.
		Though node features can be set as their identities (i.e., one-hot vectors), this will lead to another two issues:
		(2) Modeling identity of nodes cannot manage previously unseen nodes during inference and fails in inductive settings.
		(3) In real-world KGs, the number of entities are typically much larger than the number of relation types, which requires large memory for storing entity embeddings.
		
		\xhdr{Relational message passing}
		To address the above problems, a natural thought is to perform message passing over edges instead of nodes:
		\begin{align}
			m^i_e =& A \left( \big \{ s^i_{e'} \big \}_{e' \in \mathcal N (e)} \right), \label{eq:relational_mp_1} \\
			s^{i+1}_e =& U \left( s^i_e, m^i_e \right), \label{eq:relational_mp_2}
		\end{align}
		where $\mathcal N(e)$ denotes the set of neighbor edges of $e$ (i.e., edges that share at lease one common end-point with $e$) in the graph, and $s^0_e = x_e$ is the initial edge feature of $e$, i.e., the relation type.
		Therefore, Eqs. (\ref{eq:relational_mp_1}) and (\ref{eq:relational_mp_2}) are called \textit{relational message passing}.
					
		Relational message passing avoids the drawbacks of node-based message passing, however, it brings a new issue of computational efficiency when passing messages.
		To see this, we analyze the computational complexity of the two message passing schemes (proofs are given in Appendix \ref{sec:proof_1} and \ref{sec:proof_2}):		
		\begin{theorem}[Complexity of node-based message passing]
		\label{thm:1}
			Consider a graph with $N$ nodes and $M$ edges.
			The expected cost of node-based message passing (Eqs. (\ref{eq:node_based_mp_1}) and (\ref{eq:node_based_mp_2})) in each iteration is $2M + 2N$.
 		\end{theorem}

 		\begin{theorem}[Complexity of relational message passing]
 		\label{thm:2}
 			Consider a graph with $N$ nodes and $M$ edges.
			The expected cost of relational message passing (Eqs. (\ref{eq:relational_mp_1}) and (\ref{eq:relational_mp_2})) in each iteration is $N \cdot {\rm Var}[d] + \frac{4M^2}{N}$, where ${\rm Var}[d]$ is the variance of node degrees in the graph.
 		\end{theorem}
	
		\xhdr{Alternate relational message passing}
		According to the above theorems, the complexity of relational message passing is much higher than node-based message passing, especially in real-world graphs where node distribution follows the power law distribution whose variance (${\rm Var}[d]$) is extremely large due to the long tail.
		To reduce the redundant computation in relational message passing and improve its computational efficiency, we propose the following message passing scheme for KGs:
		\begin{align}
			m^i_v =& A_1 \left( \big \{ s^i_{e} \big \}_{e \in \mathcal N (v)} \right), \label{eq:alternate_mp_1} \\
			m^i_e =& A_2 \left( m^i_v, m^i_u \right), \ v, u \in \mathcal N(e), \label{eq:alternate_mp_2} \\
			s^{i+1}_e =& U \left( s^i_e, m^i_e \right). \label{eq:alternate_mp_3}
		\end{align}
		We decompose edge aggregation in Eq. (\ref{eq:relational_mp_1}) into two steps as Eqs. (\ref{eq:alternate_mp_1}) and (\ref{eq:alternate_mp_2}).
		In Eq. (\ref{eq:alternate_mp_1}), for each node $v$, we aggregate all the edges that $v$ connects to by an aggregation function $A_1(\cdot)$ and get message $m^i_v$, where $\mathcal N(v)$ denotes the set of neighbor edges for node $v$.
		Then in Eq. (\ref{eq:alternate_mp_2}), we obtain message $m^i_e$ of edge $e$ by aggregating messages from its two end-points $v$ and $u$ using function $A_2(\cdot)$, where $\mathcal N(e)$ denotes the set of neighbor nodes for edge $e$.
		The hidden state of edge $e$ is finally updated using the message $m^i_e$ as in Eq. (\ref{eq:alternate_mp_3}).
		
		An intuitive understanding of alternate relational message passing is that nodes here serve as ``distribution centers'' that collect and temporarily store the messages from their neighbor edges, then propagate the aggregated messages back to each of their neighbor edges.
		Therefore, we call Eqs. (\ref{eq:alternate_mp_1})-(\ref{eq:alternate_mp_3}) \textit{alternate relational message passing}, as messages are passed alternately between nodes and edges.
					
		The complexity of alternate relational message passing is given as follows (proof is given in Appendix \ref{sec:proof_3}):
		
		\begin{theorem}[Complexity of alternate relational message passing]
		\label{thm:3}
			Consider a graph with $N$ nodes and $M$ edges.
 			The expected cost of alternate relational message passing (Eqs. (\ref{eq:alternate_mp_1})-(\ref{eq:alternate_mp_3})) in each iteration is $6M$.
 		\end{theorem}
 		
 		From Theorem \ref{thm:3} it is clear to see that alternate relational message passing greatly reduces the time overhead and achieves the same order of complexity as node-based message passing.
 		
 		\xhdr{Remarks}
 		We present the following two remarks to provide more insight on the proposed framework:
 		
 		\begin{remark}[Relationship with belief propagation]
 		\label{remark:1}
			Alternate relational message passing is conceptually related to belief propagation (BP) \cite{yedidia2003understanding}, which is also a message-passing algorithm that passes messages between nodes and edges.
			But note that they are significantly different in:
			(1) application fields.
			BP is used to calculate the marginal distribution of unobserved variables in a graphical model, while our method aims to predict the edge type in KGs;
			(2) the purpose of using edge-node alternate message passing.
			BP uses this because of the special structure of factor graphs, while we use this to reduce the computational overhead.
		\end{remark}
		
		\begin{remark}[Utilizing node features]
		\label{remark:2}
			Though our proposed framework is claimed to only use edge features, it can be easily extended to the case where node features are present and assumed to be important, by additionally including the feature vector of node $v$ in Eq. (\ref{eq:alternate_mp_1}), i.e., $m^i_v = A_1 \left( \big \{ s^i_{e} \big \}_{e \in \mathcal N (v)}, x_v \right)$, where $x_v$ is the feature of node $v$.
			As long as node features do not contain node identities, our proposed framework is still inductive.
			We do not empirically study the performance of our method on node-feature-aware cases, because node features are unavailable for all datasets used in this paper.
			We leave the exploration of this extension to future work.
		\end{remark}

 	\subsection{Relational Context}
	\label{sec:context}
	    For a KG triplet $(h, r, t)$, relational context of $h$ and $t$ is usually highly correlated with $r$.
	    For example, if $r$ is ``graduated from'', it's reasonable to guess that the surrounding relations of $h$ are ``person.birthplace'', ``person.gender'', etc., and the surrounding relations of $t$ are ``institution.location'', ``university.founder'', ``university.president'', etc.
	    Therefore, the context of $h$ and $t$ will provide valuable clues when identifying the relation type of the edge between them, and here we use the proposed message passing method to learn from relational context.
	    
	    Denote $s_e^i$ as the hidden state of edge $e$ in iteration $i$, and $m_v^i$ as the message stored at node $v$ in iteration $i$.
	    We instantiate the alternate relational message passing in Eqs. (\ref{eq:alternate_mp_1})-(\ref{eq:alternate_mp_3}) to learn the representation of each edge:
	    \begin{align}
			m^i_v =& \sum\nolimits_{e \in \mathcal N (v)} s^i_{e}, \label{eq:context_mp_1} \\
			s^{i+1}_e =& \sigma \left( \left[ m^i_v, m^i_u, s^i_e \right] \cdot W^i + b^i \right), \ v, u \in \mathcal N(e),\label{eq:context_mp_2}
		\end{align}
		where $[\cdot]$ is the concatenation function, $W^i$, $b^i$, and $\sigma(\cdot)$ are the learnable transformation matrix, bias, and nonlinear activation function, respectively.\footnote{We shall discuss other implementations of Eqs. (\ref{eq:alternate_mp_1})-(\ref{eq:alternate_mp_3}) in Section \ref{sec:design_alternatives} and examine their performance in experiments.}
		$s_e^0 = x_e$ is initial feature of edge $e$, which can be taken as the one-hot identity vector of the relation type that $e$ belongs to.\footnote{In cases where relation types have names, initial features can also be bag-of-words (BOW) or sentence embeddings learned by language models like BERT \cite{devlin2018bert}.
		We shall investigate the performance of different initial feature types in experiments.}
		
		Relational context message passing in Eqs. (\ref{eq:context_mp_1}) and (\ref{eq:context_mp_2}) are repeated for $K$ times.
		The final message $m_h^{K-1}$ and $m_t^{K-1}$ are taken as the representation for head $h$ and tail $t$, respectively.
		We also give an illustrative example of relational context message passing in Figure \ref{fig:model}, where the red/pink edges denote the first-order/second-order contextual relations.

	\subsection{Relational Paths}
	\label{sec:path}
		We follow the discussion in Section \ref{sec:problem_formulation} and discuss how to model the term $p(t | h, r)$ or $p(h | t, r)$.
		Note that we do not consider node/edge identity in relational context message passing, which leads to a potential issue that our model is not able to identify the relative position between $h$ and $t$ in the KG.
		For example, suppose for a given entity pair $(h, t)$, $h$ is surrounded by ``person.birthplace'', ``person.gender'', etc., and $t$ is surrounded by ``institution.location'', ``university.founder'', ``university.president'', etc.
		Then it can be inferred that $h$ is probably a person and $t$ is probably a university, and there should be a relation ``graduated\_from'' between them because such a pattern appears frequently in the training data.
		However, the person may have no relationship with the university and they are far from each other in the KG.
		The reason why such false positive case happens is that relational context message passing can only detect the ``type'' of $h$ and $t$, but is not aware of their relative position in the KG.
		
		To solve this problem, we propose to explore the connectivity pattern between $h$ and $t$, which are represented by the paths connecting them in the KG.
		Specifically, a raw path from $h$ to $t$ in a KG is a sequence of entities and edges: $h(v_0) \xrightarrow{e_0} v_1 \xrightarrow{e_1} v_2 \cdots v_{L-1} \xrightarrow{e_{L-1}} t(v_L)$, in which two entities $v_i$ and $v_{i+1}$ are connected by edge $e_i$, and each entity in the path is unique.\footnote{Entities in a path are required to be unique because a loop within a path does not provide additional semantics thus should be cut off from the path.}
		The corresponding relational path $P$ is the sequence of relation types of all edges in the given raw path, i.e., $P = \left( r_{e_0}, r_{e_1}, ..., r_{e_{L-1}} \right)$, where $r_{e_i}$ is the relation type of edge $e_i$.
		Note that we do not use the identity of nodes when modeling relational paths, which is the same as for relational context.
		
		Denote $\mathcal P_{h \rightarrow t}$ as the set of all relational paths from $h$ to $t$ in the KG.
		Our next step is to define and calculate the representation of relational paths.
		In \alg, we assign an independent embedding vector $s_P$ for each relational path $P \in \mathcal P_{h \rightarrow t}$.\footnote{Other methods for calculating path representations are also possible. We shall discuss them in Section \ref{sec:design_alternatives}.}
		A potential concern here is that the number of different paths increases exponentially with the path length (there are $|r|^k$ $k$-hop paths), however, in practice we observe that in real-world KGs most paths actually do not occur (e.g., only 3.2\% of all possible paths of length 2 occur in FB15K dataset), and the number of different paths is actually quite manageable for relatively small values of $k$ ($k\le4$).
		
		An illustrative example of relational paths is shown in Figure \ref{fig:model}, where the two green arrows denote the relational paths from head entity $h$ to tail entity $t$.

	\subsection{Combining Relational Context and Paths}
	\label{sec:model}
	    For relational context, we use massage passing scheme to calculate the final message $m_h^{K-1}$ and $m_t^{K-1}$ for $h$ and $t$, which summarizes their context information, respectively.
	    $m_h^{K-1}$ and $m_t^{K-1}$ are further combined together for calculating the context of $(h, t)$ pair:
		\begin{equation}
		\label{eq:context}
			s_{(h, t)} = \sigma \left( \left[ m^{K-1}_h, m^{K-1}_t \right] \cdot W^{K-1} + b^{K-1} \right),
		\end{equation}
		where $s_{(h, t)}$ denotes the context representation of the entity pair $(h, t)$.
		Note here that Eq. (\ref{eq:context}) should only take messages of $h$ and $t$ as input without their connecting edge $r$, since the ground truth relation $r$ should be treated unobserved in the training stage.
		
		For relational paths, note that there may be a number of relational paths for a given $(h, t)$ pair, but not all paths are logically related to the predicted relation $r$, and the importance of each path also varies.
		In \alg, since we have already known the context $s_{(h, t)}$ for $(h, t)$ pair and it can be seen as prior information for paths between $h$ and $t$, we can calculate the importance scores of paths based on $s_{(h, t)}$.
		Therefore, we first calculate the attention weight of each path $P$ with respect to the context $s_{(h, t)}$:
		\begin{equation}
		\label{eq:weights}
			\alpha_P = \frac{\exp \left( {s_P}^\top s_{(h, t)} \right)}{\sum_{P \in \mathcal P_{h \rightarrow t}} \exp \left( {s_P}^\top s_{(h, t)} \right)},
		\end{equation}
		where $\mathcal P_{h \rightarrow t}$ is the set of all paths from $t$ to $t$.
		Then the attention weights are used to average representations of all paths:
		\begin{equation}
		\label{eq:attention}
			s_{h \rightarrow t} = \sum\nolimits_{P \in \mathcal P_{h \rightarrow t}} \alpha_P s_P,
		\end{equation}
		where $s_{h \rightarrow t}$ is the aggregated representation of relational paths for $(h, t)$.
		In this way, the context information $s_{(h, t)}$ is used to assist in identifying the most important relational paths.
		
		Given the relational context representation $s_{(h,t)}$ and the relational path representation $s_{h \rightarrow t}$, we can predict relations by first adding the two representation together and then taking softmax as follows:
		\begin{equation}
		    p(r|h,t)=\textsc{SoftMax} \left( s_{(h,t)} + s_{h\to t} \right).
		\end{equation}

		Our model can be trained by minimizing the loss between predictions and ground truths over the training triplets:
		\begin{equation}
		\label{eq:loss}
			\min \mathcal L = \sum_{(h, r, t) \in \mathcal D} J \big( p(r|h, t), \ r \big),
		\end{equation}
		where $\mathcal D$ is the training set and $J(\cdot)$ is the cross-entropy loss.
		
		It is worth noticing that the context representation $s_{(h, t)}$ plays two roles in the model:
		It directly contributes to the predicted relation distribution, and it also helps determine the importance of relational paths with respect to the predicted relation.

	\subsection{Discussion on Model Explainability}
	\label{sec:explain}
		Since \alg only models relations without entities, it is able to capture pure relationship among different relation types thus can naturally be used to explain for predictions.
		The explainability of \alg is two-fold:
		
		On the one hand, modeling relational context captures the correlation between contextual relations and the predicted relation, which can be used to indicate important neighbor edges for the given relation.
		For example, ``institution.location'', ``university.founder'', and ``university.president'' can be identified as important contextual relations for ``graduated from''.
		
		On the other hand, modeling relational paths captures the correlation between paths and the predicted relation, which can indicate important relational paths for the given relation.
		For example, (``schoolmate of'', ``graduated from'') can be identified as an important relational path for ``graduated from''.
		
		It is interesting to see that the explainability provided by relational paths is also connected to first-logic logical rules with the following form:
		\begin{equation*}
			B_1(h, x_1) \wedge B_2 (x_1, x_2) \wedge \cdots \wedge B_L(x_{L-1}, t) \Rightarrow r(h, t),
		\end{equation*}
		where $\bigwedge B_i$ is the conjunction of relations in a path and $r(h, t)$ is the predicted relation.
		The above example of relational path can therefore be written as the following rule:
		\begin{equation*}
		\begin{split}
		    &(h, \ \textsf{schoolmate of}, \ x) \wedge (x, \ \textsf{graduated from}, \ t) \\
		    \Rightarrow & (h, \ \textsf{graduated from}, \ t).
		\end{split}
		\end{equation*}
		Therefore, \alg can also be used to learn logical rules from KGs just as prior work \cite{galarraga2015fast, yang2017differentiable, ho2018rule, zhang2019iteratively, sadeghian2019drum}.

	\subsection{Design Alternatives}
    \label{sec:design_alternatives}
        Next we discuss several design alternatives for \alg.
        In our ablation experiments we will compare \alg with the following alternative implementations.
        
        When modeling relational context, we propose two alternatives for context aggregator, instead of the Concatenation context aggregator in Eqs. (\ref{eq:context_mp_2}) and (\ref{eq:context}):
        
        \xhdr{Mean context aggregator}
        It takes the element-wise mean of the input vectors, followed by a nonlinear transformation function:
		\begin{equation}
			s^{i+1}_e = \sigma \left( \frac{1}{3} \big( m^i_v + m^i_u + s^i_e \big) W + b \right), \ v, u \in \mathcal N(e),
		\end{equation}
		The output of Mean context aggregator is invariant to the permutation of its two input nodes, indicating that it treats the head and the tail equally in a triplet.
		
		\xhdr{Cross context aggregator}
		It is inspired by combinatorial features in recommender systems \cite{wang2019multi}, which measure the interaction of unit features (e.g., AND(gender=\textsf{female}, language=\textsf{English})).
		Note that Mean and Concatenation context aggregator simply transform messages from two input nodes separately and add them up together, without modeling the interaction between them that might be useful for link prediction.
		In Cross context aggregator, we first calculate all element-level pairwise interactions between messages from the head and the tail:
		\begin{equation}
			m^i_v {m^i_u}^\top =
			\begin{bmatrix}
   				{m^i_v}^{(1)} {m^i_u}^{(1)} & \cdots & {m^i_v}^{(1)} {m^i_u}^{(d)}\\
   				\cdots & & \cdots\\
   				{m^i_v}^{(d)} {m^i_u}^{(1)} & \cdots & {m^i_v}^{(d)} {m^i_u}^{(d)}
  			\end{bmatrix},
		\end{equation}
		where we use superscript with parentheses to indicate the element index and $d$ is the dimension of $m^i_v$ and $m^i_u$.
		Then we summarize all interactions together via flattening the interaction matrix to a vector then multiplied by a transformation matrix:
		\begin{equation}
			s^{i+1}_e = \sigma \left( {\rm flatten} \big( m^i_v {m^i_u}^\top \big) W_1^i + s^i_e W_2^i + b^i \right), \ v, u \in \mathcal N(e).
		\end{equation}
		It is worth noting that Cross context aggregator preserves the order of input nodes.

		\xhdr{Learning path representation with RNN}
		When modeling relational paths, recurrent neural network (RNN) can be used to learn the representation of relational path $P = (r_1, r_2, ...)$:
		\begin{equation}
			s_P = {\rm RNN} \left( r_1, r_2, ... \right),
		\end{equation}
		instead of directly assigning an embedding vector to $P$.
		The advantage of RNN against path embedding is that its number of parameters is fixed and does not depend on the number of relational paths.
		Another potential benefit is that RNN can hopefully capture the similarity among different relational paths.

		\xhdr{Mean path aggregator}
		When calculating the final representation of relational paths for $(h, t)$ pair, we can also simply average all the representations of paths from $h$ to $t$ instead of the Attention path aggregator in Eqs. (\ref{eq:weights}) and (\ref{eq:attention}):
		\begin{equation}
			s_{h \rightarrow t} = \sum\nolimits_{P \in \mathcal P_{h \rightarrow t}} s_P.
		\end{equation}
		Mean path aggregator can be used in the case where representation of relational context is unavailable, since it does not require attention weights as input.

\section{Experiments}
	In this section, we evaluate the proposed \alg model, and present its performance on six KG datasets.
   \subsection{Experimental Setup}
		
    	\textbf{Datasets}.
    	We conduct experiments on five standard KG benchmarks: \textbf{FB15K}, \textbf{FB15K-237}, \textbf{WN18}, \textbf{WN18RR}, \textbf{NELL995}, and one KG dataset proposed by us: \textbf{DDB14}.
    	
    	\textbf{FB15K} \cite{bordes2011learning} is from Freebase \cite{bollacker2008freebase}, a large-scale KG of general human knowledge.
    	\textbf{FB15k-237} \cite{toutanova2015observed} is a subset of FB15K where inverse relations are removed.
    	\textbf{WN18} \cite{bordes2011learning} contains conceptual-semantic and lexical relations among English words from WordNet \cite{miller1995wordnet}.
    	\textbf{WN18RR} \cite{dettmers2018convolutional} is a subset of WN18 where inverse relations are removed.
    	\textbf{NELL995} \cite{xiong2017deeppath} is extracted from the 995th iteration of the NELL system \cite{carlson2010toward} containing general knowledge.
    	
    	In addition, we present a new  dataset \textbf{DDB14} that is suitable for KG-related tasks.
    	DDB14 is collected from Disease Database\footnote{\url{http://www.diseasedatabase.com}}, which is a medical database containing terminologies and concepts such as diseases, symptoms, drugs, as well as their relationships.
    	We randomly sample two subsets of 4,000 triplets from the original one as validation set and test set, respectively.
    	
    	The statistics of the six datasets are summarized in Table \ref{table:statistics}.
    	We also calculate and present the mean and variance of node degree distribution (i.e., $\mathbb E[d]$ and ${\rm Var}[d]$) for each KG.
    	It is clear that ${\rm Var}[d]$ is large for all KGs, which empirically demonstrates that the complexity of relational message passing is fairly high, thus alternate relational message passing is necessary for real graphs.
    	
    	\begin{table}[t]
			\centering
			\small
			\setlength{\tabcolsep}{2pt}
			\begin{tabular}{c|cccccc}
				\hline
				& FB15K & FB15K-237 & WN18 & WN18RR & NELL995 & DDB14 \\
				\hline
				\#nodes & 14,951 & 14,541 & 40,943 & 40,943 & 63,917 & 9,203 \\
				\#relations & 1,345 & 237 & 18 & 11 & 198 & 14 \\
				\#training & 483,142 & 272,115 & 141,442 & 86,835 & 137,465 & 36,561 \\
				\#validation & 50,000 & 17,535 & 5,000 & 3,034 & 5,000 & 4,000 \\
				\#test & 59,071 & 20,466 & 5,000 & 3,134 & 5,000 & 4,000 \\
				$\mathbb E[d]$ & 64.6 & 37.4 & 6.9 & 4.2 & 4.3 & 7.9 \\
				${\rm Var}[d]$ & 32,441.8 & 12,336.0 & 236.4 & 64.3 & 750.6 & 978.8 \\
				\hline
			\end{tabular}
			\vspace{0.05in}
			\caption{Statistics of all datasets. $\mathbb E[d]$ and ${\rm Var}[d]$ are mean and variance of the node degree distribution, respectively.}
			\label{table:statistics}
			\vspace{-0.2in}
		\end{table}

    	\xhdr{Baselines}
    	We compare \alg with several state-of-the-art models, including \textbf{TransE} \cite{bordes2013translating}, \textbf{ComplEx} \cite{trouillon2016complex}, \textbf{DistMult} \cite{yang2015embedding}, \textbf{RotatE} \cite{sun2019rotate}, \textbf{SimplE} \cite{kazemi2018simple}, \textbf{QuatE} \cite{zhang2019quaternion}, and \textbf{DRUM} \cite{sadeghian2019drum}. The first six models are embedding-based methods, while DRUM only uses relational paths to make prediction.
    	The implementation details of baselines (as well as our method) is provided in Appendix \ref{sec:implementation}.
    	
    	We also conduct extensive ablation study and propose two reduced versions of our model, \textbf{\textsc{Con}} and \textbf{\textsc{Path}}, which only use relational context and relational paths, respectively, to test the performance of the two components separately.
    	
    	The number of parameters of each model on DDB14 are shown in Table \ref{table:params}.
    	The result demonstrates that \alg is much more storage-efficient than embedding-based methods, since it does not need to calculate and store entity embeddings.
    	
    	\begin{table}[h]
			\centering
			\small
			\setlength{\tabcolsep}{2pt}
            \begin{tabular}{c|ccccccc}
               \hline
               Method & TransE & ComplEx & DisMult & RotatE & SimplE & QuatE  & \alg \\
               \hline
                \#param. & 3.7M & 7.4M & 3.7M & 7.4M  & 7.4M & 14.7M & 0.06M \\
                \hline
            \end{tabular}
            \vspace{0.05in}
            \caption{Number of parameters of all models on DDB14.}
            \vspace{-0.2in}
            \label{table:params}
        \end{table}

            \begin{table*}[t]
    			\centering
    			\small
    			\setlength{\tabcolsep}{2pt}
    			\begin{tabular}{c|ccc|ccc|ccc|ccc|ccc|ccc}
    				\hline
    				& \multicolumn{3}{c|}{FB15K} & \multicolumn{3}{c|}{FB15K-237} & \multicolumn{3}{c|}{WN18} & \multicolumn{3}{c|}{WN18RR} & \multicolumn{3}{c|}{NELL995} & \multicolumn{3}{c}{DDB14} \\
            		& \multicolumn{1}{c}{MRR} & \multicolumn{1}{c}{Hit@1} & \multicolumn{1}{c|}{Hit@3} & \multicolumn{1}{c}{MRR} & \multicolumn{1}{c}{Hit@1} & \multicolumn{1}{c|}{Hit@3} & \multicolumn{1}{c}{MRR} & \multicolumn{1}{c}{Hit@1} & \multicolumn{1}{c|}{Hit@3} & \multicolumn{1}{c}{MRR} & \multicolumn{1}{c}{Hit@1} & \multicolumn{1}{c|}{Hit@3} & \multicolumn{1}{c}{MRR} & \multicolumn{1}{c}{Hit@1} & \multicolumn{1}{c|}{Hit@3} & \multicolumn{1}{c}{MRR} & \multicolumn{1}{c}{Hit@1} & \multicolumn{1}{c}{Hit@3} \\
            		\hline
            		TransE & 0.962 & 0.940 & 0.982 & 0.966 & 0.946 & 0.984 & 0.971 & 0.955 & 0.984 & 0.784 & 0.669 & 0.870 & \underline{0.841} & \underline{0.781} & \underline{0.889} & \underline{0.966} & \underline{0.948} & 0.980 \\
            		ComplEx & 0.901 & 0.844 & 0.952 & 0.924 & 0.879 & 0.970 & \underline{0.985} & \underline{0.979} & \underline{0.991} & 0.840 & 0.777 & 0.880 & 0.703 & 0.625 & 0.765 & 0.953 & 0.931 & 0.968 \\
            		DistMult & 0.661 & 0.439 & 0.868 & 0.875 & 0.806 & 0.936 & 0.786 & 0.584 & 0.987 & 0.847 & \underline{0.787} & 0.891 & 0.634 & 0.524 & 0.720 & 0.927 & 0.886 & 0.961 \\
            		RotatE & 0.979 & 0.967 & 0.986 & 0.970 & 0.951 & 0.980 & 0.984 & \underline{0.979} & 0.986 & 0.799 & 0.735 & 0.823 & 0.729 & 0.691 & 0.756 & 0.953 & 0.934 & 0.964 \\
            		SimplE & \underline{0.983}  & \underline{0.972} & \underline{0.991} & 0.971 & 0.955 & 0.987 & 0.972 & 0.964 & 0.976 & 0.730 & 0.659 & 0.755 & 0.716 & 0.671 & 0.748 & 0.924 & 0.892 & 0.948 \\
            		QuatE & \underline{0.983} & \underline{0.972} & \underline{0.991} & \underline{0.974} & \underline{0.958} & \underline{0.988} & 0.981 & 0.975 & 0.983 & 0.823 & 0.767 & 0.852 & 0.752 & 0.706 & 0.783 & 0.946 & 0.922 & 0.962 \\
            		DRUM & 0.945 & 0.945 & 0.978 & 0.959 & 0.905 & 0.958 & 0.969 & 0.956 & 0.980 & \underline{0.854} & 0.778 & \underline{0.912} & 0.715 & 0.640 & 0.740 & 0.958 & 0.930 & \underline{0.987} \\
            		\hline
            		\tabincell{c}{\textsc{Con}} & \tabincell{c}{0.962\\[-1.2ex] \footnotesize $\pm$ 0.000} & \tabincell{c}{0.934\\[-1.2ex] \footnotesize $\pm$ 0.000} & \tabincell{c}{0.988\\[-1.2ex] \footnotesize $\pm$ 0.000} & \tabincell{c}{0.978\\[-1.2ex] \footnotesize $\pm$ 0.000} & \tabincell{c}{0.961\\[-1.2ex] \footnotesize $\pm$ 0.001} & \tabincell{c}{\textbf{0.995}\\[-1.2ex] \footnotesize $\pm$ 0.000} & \tabincell{c}{0.960\\[-1.2ex] \footnotesize $\pm$ 0.002} & \tabincell{c}{0.927\\[-1.2ex] \footnotesize $\pm$ 0.005} & \tabincell{c}{0.992\\[-1.2ex] \footnotesize $\pm$ 0.001} & \tabincell{c}{0.943\\[-1.2ex] \footnotesize $\pm$ 0.002} & \tabincell{c}{0.894\\[-1.2ex] \footnotesize $\pm$ 0.004} & \tabincell{c}{0.993\\[-1.2ex] \footnotesize $\pm$ 0.003} & \tabincell{c}{0.875\\[-1.2ex] \footnotesize $\pm$ 0.003} & \tabincell{c}{0.815\\[-1.2ex] \footnotesize $\pm$ 0.004} & \tabincell{c}{0.928\\[-1.2ex] \footnotesize $\pm$ 0.003} & \tabincell{c}{0.977\\[-1.2ex] \footnotesize $\pm$ 0.000} & \tabincell{c}{0.961\\[-1.2ex] \footnotesize $\pm$ 0.001} & \tabincell{c}{0.994\\[-1.2ex] \footnotesize $\pm$ 0.001} \\
            		\tabincell{c}{\textsc{Path}} & \tabincell{c}{0.937\\[-1.2ex] \footnotesize $\pm$ 0.001} & \tabincell{c}{0.918\\[-1.2ex] \footnotesize $\pm$ 0.001} & \tabincell{c}{0.951\\[-1.2ex] \footnotesize $\pm$ 0.001} & \tabincell{c}{0.972\\[-1.2ex] \footnotesize $\pm$ 0.001} & \tabincell{c}{0.957\\[-1.2ex] \footnotesize $\pm$ 0.001} & \tabincell{c}{0.986\\[-1.2ex] \footnotesize $\pm$ 0.001} & \tabincell{c}{0.981\\[-1.2ex] \footnotesize $\pm$ 0.000} & \tabincell{c}{0.971\\[-1.2ex] \footnotesize $\pm$ 0.005} & \tabincell{c}{0.989\\[-1.2ex] \footnotesize $\pm$ 0.001} & \tabincell{c}{0.933\\[-1.2ex] \footnotesize $\pm$ 0.000} & \tabincell{c}{0.897\\[-1.2ex] \footnotesize $\pm$ 0.001} & \tabincell{c}{0.961\\[-1.2ex] \footnotesize $\pm$ 0.001} & \tabincell{c}{0.737\\[-1.2ex] \footnotesize $\pm$ 0.001} & \tabincell{c}{0.685\\[-1.2ex] \footnotesize $\pm$ 0.002} & \tabincell{c}{0.764\\[-1.2ex] \footnotesize $\pm$ 0.002} & \tabincell{c}{0.969\\[-1.2ex] \footnotesize $\pm$ 0.000} & \tabincell{c}{0.948\\[-1.2ex] \footnotesize $\pm$ 0.001} & \tabincell{c}{0.991\\[-1.2ex] \footnotesize $\pm$ 0.000}  \\
            		\alg & \tabincell{c}{\textbf{0.984}\\[-1.2ex] \footnotesize $\pm$ 0.001} & \tabincell{c}{\textbf{0.974}\\[-1.2ex] \footnotesize $\pm$ 0.002} & \tabincell{c}{\textbf{0.995}\\[-1.2ex] \footnotesize $\pm$ 0.001} & \tabincell{c}{\textbf{0.979}\\[-1.2ex] \footnotesize $\pm$ 0.000} & \tabincell{c}{\textbf{0.964}\\[-1.2ex] \footnotesize $\pm$ 0.001} & \tabincell{c}{0.994\\[-1.2ex] \footnotesize $\pm$ 0.001} & \tabincell{c}{\textbf{0.993}\\[-1.2ex] \footnotesize $\pm$ 0.001} & \tabincell{c}{\textbf{0.988}\\[-1.2ex] \footnotesize $\pm$ 0.001} & \tabincell{c}{\textbf{0.998}\\[-1.2ex] \footnotesize $\pm$ 0.000} & \tabincell{c}{\textbf{0.974}\\[-1.2ex] \footnotesize $\pm$ 0.001} & \tabincell{c}{\textbf{0.954}\\[-1.2ex] \footnotesize $\pm$ 0.002} & \tabincell{c}{\textbf{0.994}\\[-1.2ex] \footnotesize $\pm$ 0.000} & \tabincell{c}{\textbf{0.896}\\[-1.2ex] \footnotesize $\pm$ 0.001} & \tabincell{c}{\textbf{0.844}\\[-1.2ex] \footnotesize $\pm$ 0.004} & \tabincell{c}{\textbf{0.941}\\[-1.2ex] \footnotesize $\pm$ 0.004} & \tabincell{c}{\textbf{0.980}\\[-1.2ex] \footnotesize $\pm$ 0.000} & \tabincell{c}{\textbf{0.966}\\[-1.2ex] \footnotesize $\pm$ 0.001} & \tabincell{c}{\textbf{0.995}\\[-1.2ex] \footnotesize $\pm$ 0.000} \\
            		\hline
				\end{tabular}
				\vspace{0.05in}
				\caption{Results of relation prediction on all datasets. Best results are highlighted in bold, and best results of baselines are highlighted with underlines.}
				\label{table:result_1}
				\vspace{-0.1in}
			\end{table*}
    	
    	\xhdr{Evaluation Protocol}
    	We evaluate all methods on relation prediction, i.e., for a given entity pair $(h, t)$ in the test set, we rank the ground-truth relation type $r$ against all other candidate relation types.
    	It is worth noticing that most baselines are originally designed for head/tail prediction, therefore, their negative sampling strategy is to corrupt the head or the tail for a true triple $(h, r, t)$, i.e., replacing $h$ or $t$ with a randomly sampled entity $h'$ or $t'$ from KGs, and using $(h', r, t)$ or $(h, r, t')$ as the negative sample.
		In relation prediction, since the task is to predict the missing relation for a given pair $(h, t)$, we modify the negative sampling strategy accordingly by corrupting the relation $r$ of each true triplet $(h, r, t)$, and use $(h, r', t)$ as the negative sample where $r'$ is randomly sampled from the set of relation types.
		This new negative sampling strategy can indeed improve the performance of baselines in relation prediction.
    	
    	We use \textbf{MRR} (mean reciprocal rank) and \textbf{Hit@1, 3} (hit ratio with cut-off values of 1 and 3) as evaluation metrics.

	\subsection{Main Results}
		\textbf{Comparison with baselines}.
		The results of relation prediction on all datasets are reported in Table \ref{table:result_1}.
		In general, our method outperforms all baselines on all datasets.
		Specifically, the absolute Hit@1 gain of \alg against the best baseline in relation prediction task are $0.2\%$, $0.6\%$, $0.9\%$, $16.7\%$, $6.3\%$, and $1.8\%$ in the six datasets, respectively.
		The improvement is rather significant for WN18RR and NELL995, which are exactly the two most sparse KGs according to the average node degree shown in Table \ref{table:statistics}.
		This empirically demonstrates that \alg maintains great performance for sparse KGs, and this is probably because \alg has much fewer parameters than baselines and is less prone to overfitting.
		In contrast, performance gain of \alg on FB15K is less significant, which may be because the density of FB15K is very high so that it is much easier for baselines to handle.

		In addition, the results also demonstrate the stability of \alg as we observe that most of the standard deviations are quite small.
		
		Results in Tables \ref{table:result_1} also show that, in many cases \textsc{Con} or \textsc{Path} can already beat most baselines.
		Combining relational context and relational paths together usually leads to even better performance.

		\xhdr{Inductive KG completion}
			We also examine the performance of our method in inductive KG completion.
			We randomly sample a subset of nodes that appears in the test set, then remove these nodes along with their associated edges from the training set.
			The remaining training set is used to train the models, and we add back the removed edges during evaluation.
			The evaluation transforms from fully conductive to fully inductive when the ratio of removed nodes increases from 0 to 1.
			The results of \alg, DistMult, and RotatE on relation prediction task are plotted in Figure \ref{fig:inductive}.
			We observe that the performance of our method decreases slightly in fully inductive setting (from 0.954 to 0.922), while DistMult and RotatE fall to a ``randomly guessing'' level.
			This is because the two baselines are embedding-based models that rely on modeling node identity, while our method does not consider node identity thus being naturally generalizable to inductive KG completion.

			\begin{figure*}
  				\begin{minipage}[t]{0.3\linewidth}
  				    \centering 
    				\includegraphics[width=\textwidth]{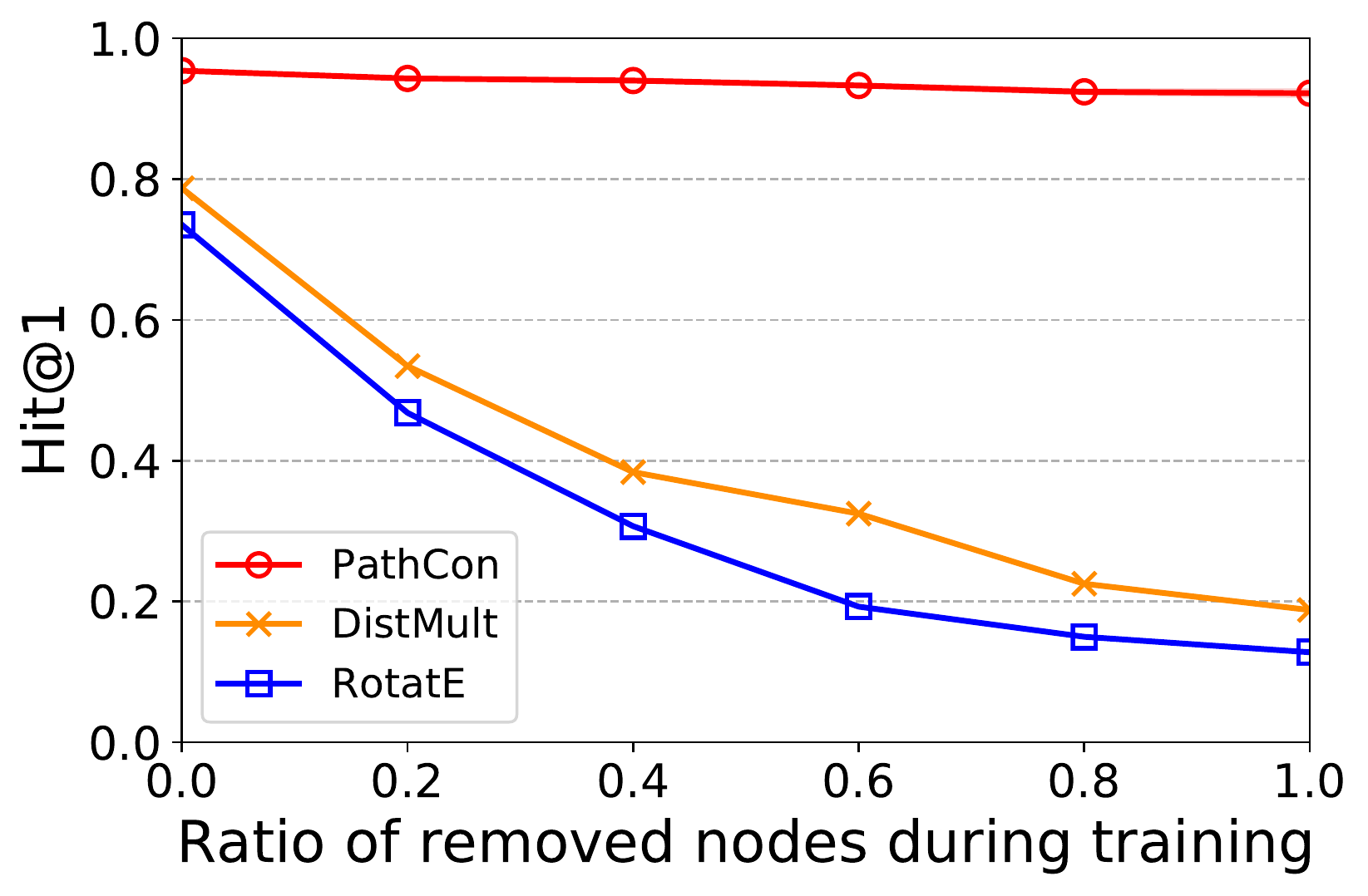}
    				\vspace{-0.2in}
    				\caption{Results of inductive KG completion on WN18RR.} 
    				\label{fig:inductive} 
  				\end{minipage}
  				\hfill
  				\begin{minipage}[t]{0.3\linewidth} 
    				\centering 
    				\includegraphics[width=\textwidth]{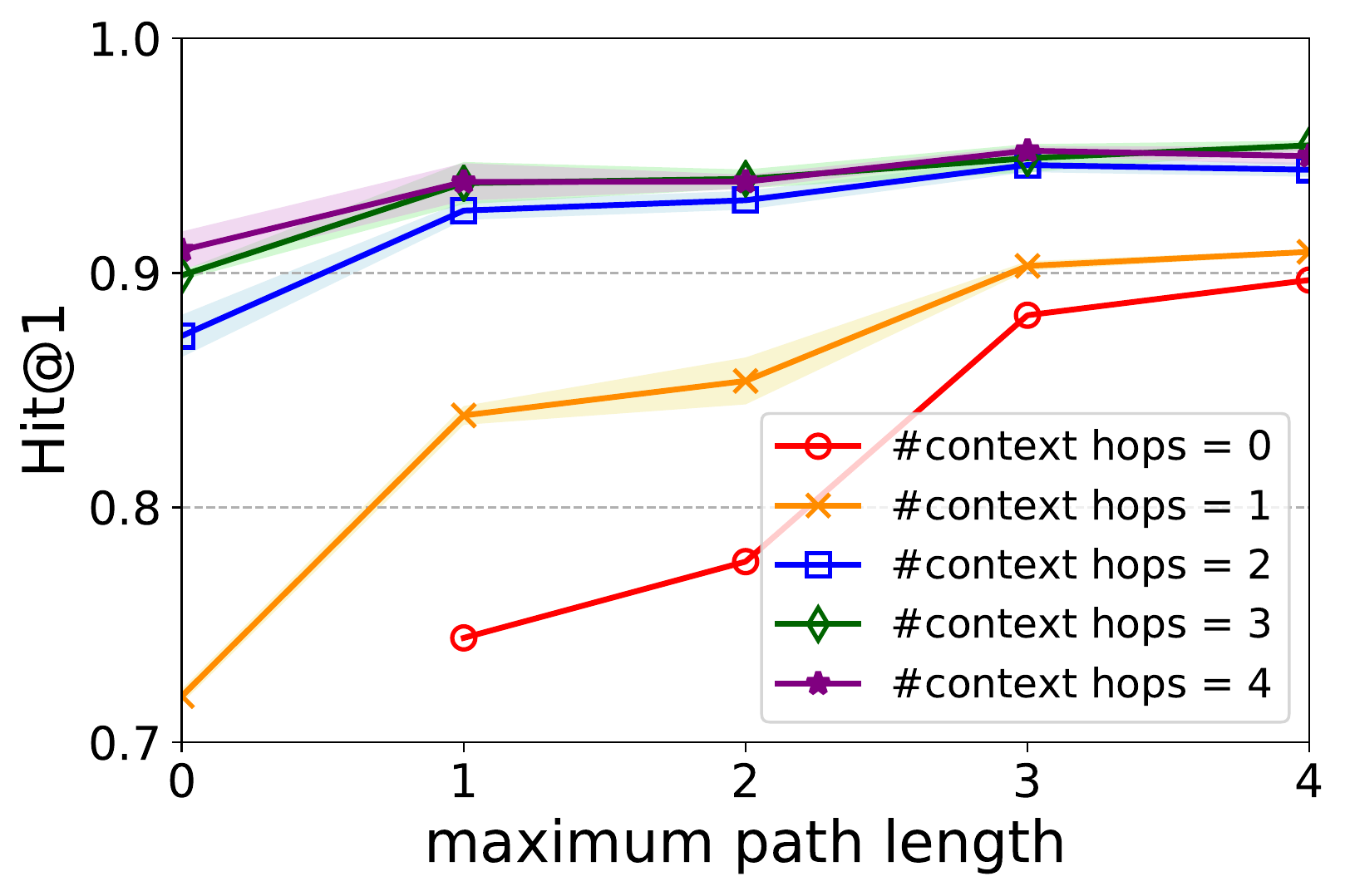}
    				\vspace{-0.2in}
    				\caption{Results of \alg with different hops/length on WN18RR.} 
    				\label{fig:layers} 
  				\end{minipage}
  				\hfill
  				\begin{minipage}[t]{0.3\linewidth} 
    				\centering 
    				\includegraphics[width=\textwidth]{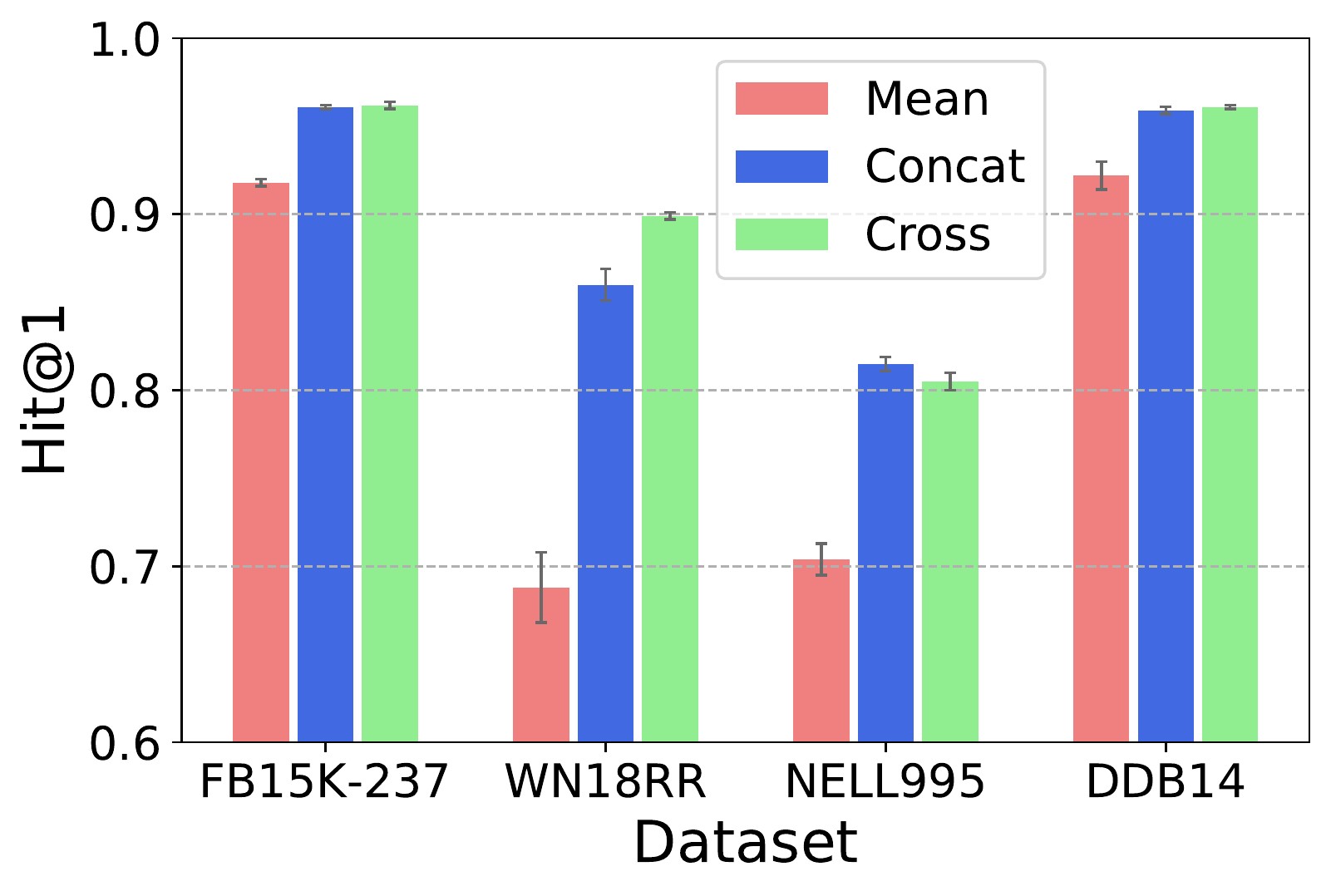}
    				\vspace{-0.2in}
    				\caption{Results of \textsc{Con} with different context aggregators.} 
    				\label{fig:neighbor_agg} 
  				\end{minipage} 
			\end{figure*}

		\begin{figure}
  			\begin{minipage}[t]{0.47\linewidth} 
    			\centering 
    			\includegraphics[width=\textwidth]{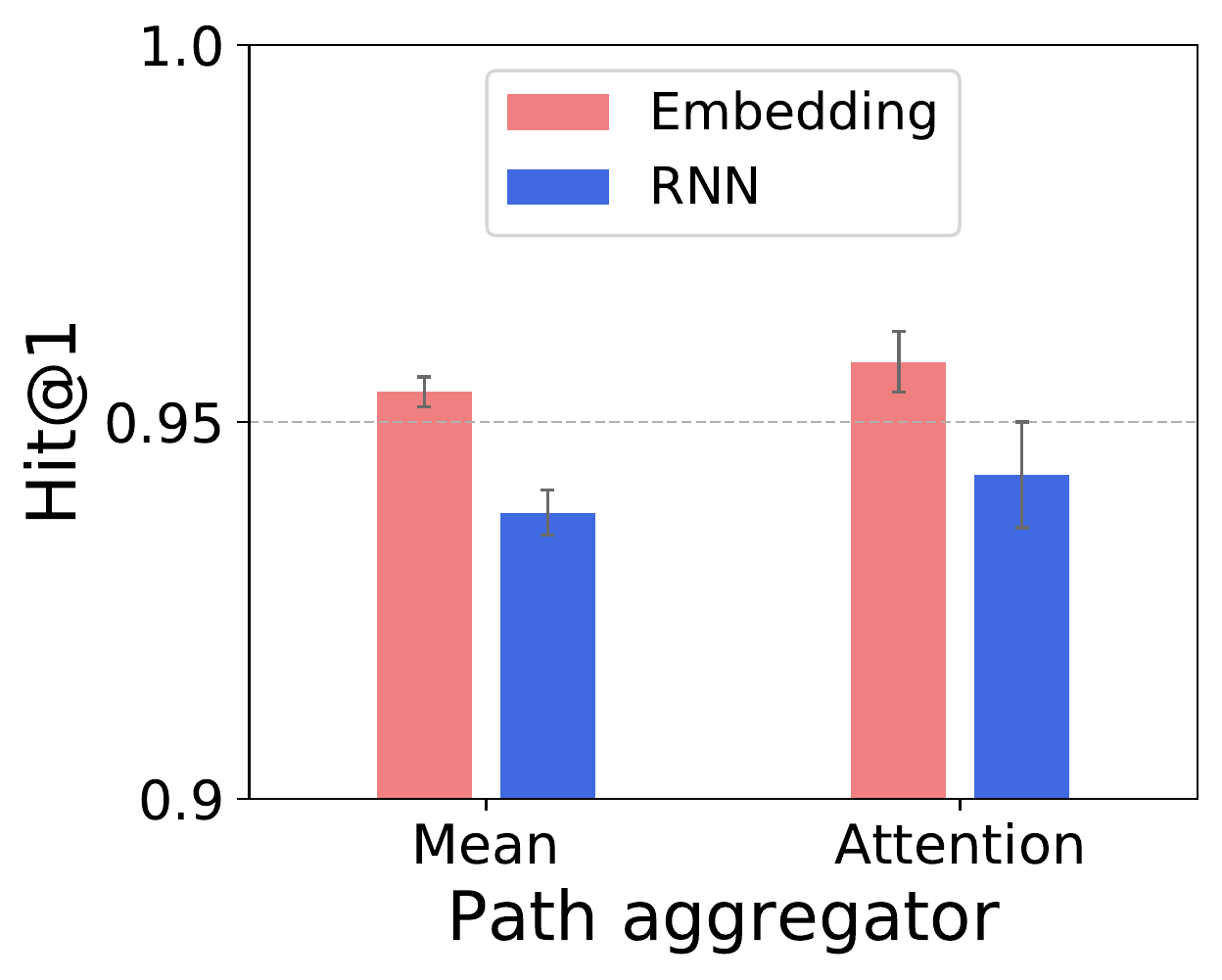}
    			\vspace{-0.2in}
    			\caption{Results of \alg with different path representation types and path aggregators on WN18RR.} 
    			\label{fig:path_agg_updater} 
  			\end{minipage}
  			\hfill
  			\begin{minipage}[t]{0.46\linewidth} 
    			\centering 
    			\includegraphics[width=\textwidth]{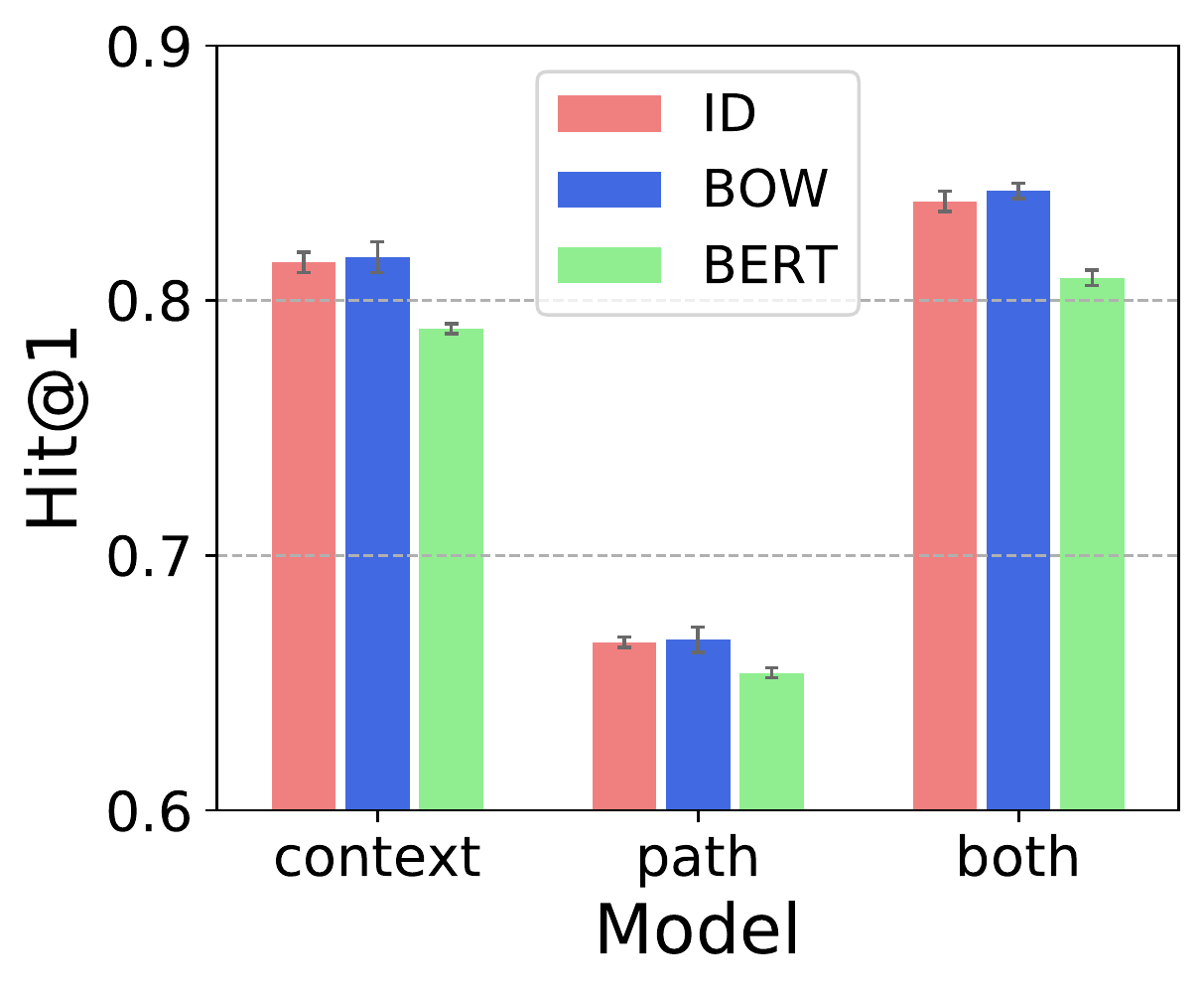}
    			\vspace{-0.2in}
    			\caption{Results of \textsc{Con}, \textsc{Path}, and \alg with different initial features of relations on NELL995.} 
    			\label{fig:edge_feature} 
  			\end{minipage} 
		\end{figure}

	\subsection{Model Variants}
	\label{sec:model_variants}
		\textbf{The number of context hops and maximum path length}.
		We investigate the sensitivity of our model to the number of context hops and maximum path length.
		We vary the two numbers from 0 to 4 (0 means the corresponding module is not used), and report the results of all combinations (without (0, 0)) on WN18RR in Figure \ref{fig:layers}.
		It is clear to see that increasing the number of context hops and maximum path length can significantly improve the result when they are small, which demonstrates that including more neighbor edges or counting longer paths does benefit the performance.
		But the marginal benefit is diminishing with the increase of layer numbers.
		Similar trend is observed on other datasets too.

		\xhdr{Context aggregators}
		We study how different implementations of context aggregator affect the model performance.
		The results of Mean, Concat, and Cross context aggregator on four datasets are shown in Figure \ref{fig:neighbor_agg} (results on FB15K and WN18 are omitted as they are similar to FB15K-237 and WN18RR, respectively).
		The results show that Mean performs worst on all datasets, which indicates the importance of node orders when aggregating features from nodes to edges.
		It is also interesting to notice that the performance comparison between Concat and Cross varies on different datasets:
		Concat is better than Cross on NELL995 and is worse than Cross on WN18RR, while their performance is on par on FB15K-237 and DDB14.
		However, note that a significant defect of Cross is that it has much more parameters than Concat, which requires more running time and memory resource.

		\xhdr{Path representation types and path aggregators}
		We implement four combinations of path representation types and path aggregators: Embedding+Mean, Embedding+Attention, RNN+Mean, and RNN+Attention, of which the results are presented in Figure \ref{fig:path_agg_updater}.
		Different from context aggregators, results on the six datasets are similar for path representation types and path aggregators, so we only report the results on WN18RR.
		We find that Embedding is consistently better than RNN, which is probably because the length of relational paths are generally short (no more than 4 in our experiments), so RNN can hardly demonstrate its strength in modeling sequences.
		The results also show that Attention aggregator performs slightly better than Mean aggregator.
		This demonstrates that the contextual information of head and tail entities indeed helps identify the importance of relational paths.

		\xhdr{Initial edge features}
		Here we examine three types of initial edge features: identity, BOW, and BERT embedding of relation types.
		We choose to test on NELL995 because its relation names consist of relatively more English words thus are semantically meaningful (e.g., ``organization.headquartered.in.state.or.province'').
		The results are reported in Figure \ref{fig:edge_feature}, which shows that BOW features are slightly better than identity, but BERT embeddings perform significantly worse than the other two.
		We attribute this finding to that
		BERT embeddings are better at identifying semantic relationship among relation types, but our model aims to learn the mapping from BERT embeddings of context/paths to the identity of predicted relation types.
		In other words, BERT may perform better if the predicted relation types are also represented by BERT embeddings, so that this mapping is learned within the embedding space.
		We leave the exploration as future work.

	\subsection{Case Study on Model Explainabilty}
		We choose FB15K-237 and DDB14 as the datasets to show the explainability of \alg.
		The number of context hops is set to 1 and the maximum path length is set to 2.
		When training is completed, we choose three relations from each dataset and list the most important relational context/paths to them based on the transformation matrix of the context/path aggregator.
		The results are presented in Table \ref{table:rule}, from which we find that most of the identified context/paths are logically meaningful.
		For example, ``education campus of'' can be inferred by ``education institution in'', and ``is associated with'' is found to be a transitive relation.
		In addition, more visualized results and discussion on DDB14 dataset are included in Appendix \ref{sec:ddb14_explain}.

		\begin{table*}[t]
			\centering
			\small
			\setlength{\tabcolsep}{3pt}
			\begin{tabular}{c|c|c|c}
				\hline
				& Predicted relation & Important relational context & Important relational paths \\
				\hline
				\multirow{3}{*}{FB15K-237} & award winner & award honored for, award nominee & (award nominated for), (award winner, award category) \\
				& film written by & film release region & (film edited by), (film crewmember) \\
				& education campus of & education major field of study & (education institution in) \\
				\hline
				\multirow{3}{*}{DDB14} & may cause & may cause, belongs to the drug family of & (is a risk factor for), (see also, may cause) \\
				& is associated with & is associated with, is a risk factor for & (is associated with, is associated with) \\
				& may be allelic with & may be allelic with, belong(s) to the category of & (may cause, may cause), (may be allelic with, may be allelic with) \\
				\hline
			\end{tabular}
			\vspace{0.05in}
			\caption{Examples of important context/paths identified by \alg on FB15K-237 and DDB14.}
			\label{table:rule}
			\vspace{-0.2in}
		\end{table*}

\section{Related Work}

	\subsection{Knowledge Graph Completion}
		KGs provide external information for a variety of downstream tasks such as recommender systems \cite{wang2018dkn, wang2018ripplenet, wang2019exploring} and semantic analysis \cite{wang2018shine}.
		Most existing methods of KG completion are based on embeddings, which normally assign an embedding vector to each entity and relation in the continuous embedding space and train the embeddings based on the observed facts.
		One line of KG embedding methods is \textit{translation-based}, which treat entities as points in a continuous space and each relation translates the entity point.
		The objective is that the translated head entity should be close to the tail entity in real space \cite{bordes2013translating}, complex space \cite{sun2019rotate}, or quaternion space \cite{zhang2019quaternion}, which have shown capability to handle multiple relation patterns and achieve state-of-the-art result.
		Another line of work is \textit{multi-linear} or \textit{bilinear models}, where they calculate the semantic similarity by matrix or vector dot product in real \cite{yang2015embedding} or complex space \cite{trouillon2016complex}.
		Besides, several embedding-based methods explore the architecture design that goes beyond point vectors \cite{socher2013reasoning, dettmers2018convolutional}.
		However, these embedding-based models fail to predict links in inductive setting, neither can they discover any rules that explain the prediction.
		
		

	\subsection{Graph Neural Networks}
		Existing GNNs generally follow the idea of neural message passing \cite{gilmer2017neural} that consists of two procedures: propagation and aggregation.
		Under this framework, several GNNs are proposed that take inspiration from convolutional neural networks \cite{duvenaud2015convolutional, hamilton2017inductive, kipf2017semi, wang2020unifying}, recurrent neural networks \cite{li2015gated}, and recursive neural networks \cite{bianchini2001processing}. 
        However, these methods use node-based message passing, while we propose passing messages based on edges in this work. 
		
		There are two GNN models conceptually connected to our idea of identifying relative position of nodes in a graph.
		DEGNN \cite{li2020distance} captures the distance between the node set whose representation is to be learned and each node in the graph, which is used as extra node attributes or as controllers of message aggregation in GNNs.
		SEAL \cite{zhang2018link} labels nodes with their distance to two nodes $a$ and $b$ when predicting link existence between $(a, b)$.
		In contrast, we use relational paths to indicate the relative position of two nodes.
		
		Researchers also tried to apply GNNs to knowledge graphs.
		For example, Schlichtkrull \textit{et al.} \cite{schlichtkrull2018modeling} use GNNs to model the entities and relations in KGs, however, they are limited in that they did not consider the relational paths and cannot predict in inductive settings. 
		Wang \textit{et al.} \cite{wang2019knowledge_a, wang2019knowledge_b} use GNNs to learn entity embeddings in KGs, but their purpose is to use the learned embeddings to enhance the performance of recommender systems rather than KG completion.


\section{Conclusion and Future Work}
	We propose \alg for KG completion.
	\alg considers two types of subgraph structure in KGs, i.e., contextual relations of the head/tail entity and relational paths between head and tail entity.
	We show that both relational context and relational paths are critical to relation prediction, and they can be combined further to achieve state-of-the-art performance.
	Moreover, \alg is also shown to be inductive, storage-efficient, and explainable.
	
	We point out four directions for future work.
	First, as we discussed in Remark \ref{remark:2}, it is worth studying the empirical performance of \alg on node-feature-aware KGs.
	Second, as we discussed in Section \ref{sec:model_variants}, designing a model that can better take advantage of pre-trained word embeddings is a promising direction;
	Third, it is worth investigating why RNN does not perform well, and whether we can model relational paths better;
	Last, it is interesting to study if the context representation and path representation can be assembled in a more principled way.
	
\xhdr{Acknowledgements}
This research has been supported in part by DARPA, ARO, NSF, NIH, Stanford Data Science Initiative, Wu Tsai Neurosciences Institute, Chan Zuckerberg Biohub, Amazon, JPMorgan Chase, Docomo, Hitachi, Intel, JD.com, KDDI, NVIDIA, Dell, Toshiba, Visa, and UnitedHealth Group.

\bibliographystyle{ACM-Reference-Format}
\bibliography{reference}

\renewcommand\thesubsection{\Alph{subsection}}

\clearpage

\section*{Appendix}
	\subsection{Proof of Theorem \ref{thm:1}}
	\label{sec:proof_1}
		\begin{proof}
			In each iteration of node-based message passing:
			
			The aggregation (Eq. (\ref{eq:node_based_mp_1})) is performed for $N$ times, and each aggregation takes $\mathbb E[d] = \frac{2M}{N}$ elements as input in expectation, where $\mathbb E[d]$ is the expected node degree.
			Therefore, the expected cost of aggregation in each iteration is $N \cdot \mathbb E[d] = 2M$;
			
			The update (Eq. (\ref{eq:node_based_mp_2})) is performed for $N$ times, and each update takes 2 elements as input.
			Therefore, the cost of update in each iteration is $2N$.
			
			In conclusion, the expected cost of node-based message passing in each iteration is $2M + 2N$.
		\end{proof}
		
		 \begin{figure*}[t]
    		\centering 
    		\includegraphics[width=0.9\textwidth]{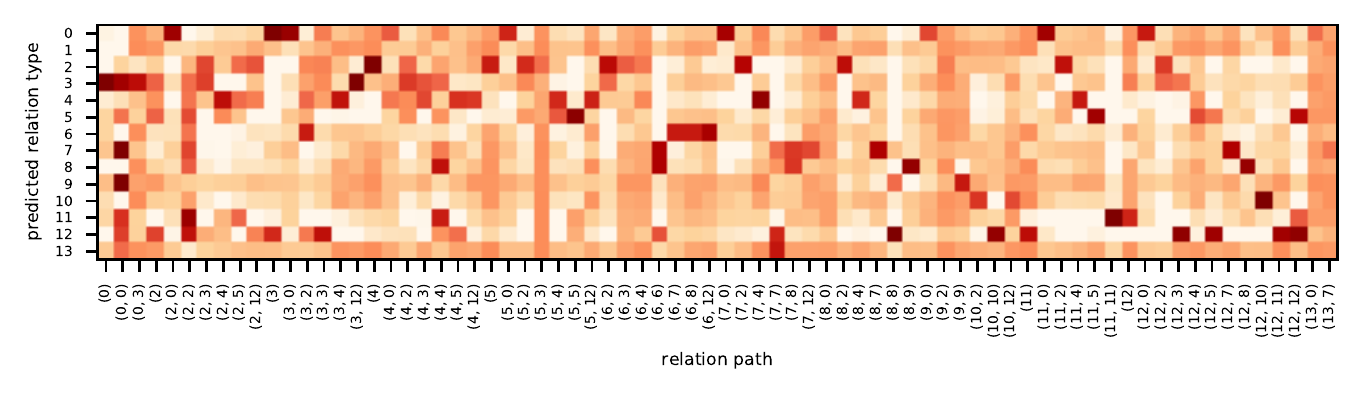}
    		\vspace{-0.2in}
    		\caption{The learned correlation between all relational paths with length $\leq 2$ and the predicted relations on DDB14.}
    		\label{fig:heatmap_path} 
		\end{figure*}

	\subsection{Proof of Theorem \ref{thm:2}}
	\label{sec:proof_2}
			For relational message passing, it actually passes messages on the \textit{line graph} of the original graph.
			The line graph of a given graph $\mathcal G$, denoted by $L(\mathcal G)$, is a graph such that each node of $L(\mathcal G)$ represents an edge of $\mathcal G$, and two nodes of $L(\mathcal G)$ are adjacent if and only if their corresponding edges share a common endpoint in $\mathcal G$.
			We show by the following lemma that the line graph is much \textit{larger} and \textit{denser} than the original graph:
			
			\begin{lemma}
				\label{lemma:1}
				The number of nodes in line graph $L(\mathcal G)$ is $M$, and the expected node degree of $L(\mathcal G)$ is
				\begin{equation}
					\mathbb E_{L(\mathcal G)}[d] = \frac{N \cdot {\rm Var}_{\mathcal G}[d]}{M} + \frac{4M}{N} - 2,
				\end{equation}
				where ${\rm Var}_{\mathcal G}[d]$ is the variance of node degrees in $\mathcal G$.
			\end{lemma}
	
			\begin{proof}
				It is clear that the number of nodes in line graph $L(\mathcal G)$ is $M$ because each node in $L(\mathcal G)$ corresponds to an edge in $\mathcal G$.
				We now prove that the expected node degree of $L(\mathcal G)$ is $\mathbb E_{L(\mathcal G)} [d] = \frac{N \cdot {\rm Var}_{\mathcal G}[d]}{M} + \frac{4M}{N} - 2$.
		
				Let's first count the number of edges in $L(\mathcal G)$.
				According to the definition of line graph, each edge in $L(\mathcal G)$ corresponds to an unordered pair of edges in $\mathcal G$ connecting to a same node;
				On the other hand, each unordered pair of edges in $\mathcal G$ that connect to a same node also determines an edge in $L(\mathcal G)$.
				Therefore, the number of edges in $L(\mathcal G)$ equals the number of all unordered pairs of edges connecting to a same node:
				\begin{equation*}
					\# \ edges \ in \ L(\mathcal G) = \sum_i \binom{d_i}{2} = \sum_i \frac{d_i (d_i - 1)}{2} = \frac{1}{2} \sum_i d_i^2 - M,
				\end{equation*}
				where $d_i$ is the degree of node $v_i$ in $\mathcal G$ and $M = 2 \sum_i d_i$ is the number of edges.
				Then the the expected node degree of $L(\mathcal G)$ is
				\begin{equation*}
				\begin{split}
					\mathbb E_{L(\mathcal G)} [d] =& 2 \cdot \frac{\# \ edges \ in \ L(\mathcal G)}{\# \ nodes \ in \ L(\mathcal G)} = \frac{\sum_i d_i^2 - 2M}{M} \\
					=& \frac{N \cdot \mathbb E_{\mathcal G}[d^2]}{M} - 2 = \frac{N \left( {\rm Var}_{\mathcal G}[d] + \mathbb E^2_{\mathcal G}[d] \right)}{M} - 2 \\
					=& \frac{N \cdot {\rm Var}_{\mathcal G}[d] + N \left( \frac{2M}{N} \right)^2}{M} - 2 \\
					=& \frac{N \cdot {\rm Var}_{\mathcal G}[d]}{M} + \frac{4M}{N} - 2.
				\end{split}
				\end{equation*}
			\end{proof}
	
			From Lemma \ref{lemma:1} it is clear to see that $\mathbb E_{L(\mathcal G)} [d]$ is at least twice of $\mathbb E_{\mathcal G} [d] = \frac{2M}{N}$, i.e. the expected node degree of the original graph $\mathcal G$, since ${\rm Var}_{\mathcal G}[d] \geq 0$ ($-2$ is omitted).
			Unfortunately, in real-world graphs (including KGs), node degrees vary significantly, and they typically follow the power law distribution whose variance is extremely large due to the long tail (this is empirically justified in Table \ref{table:statistics}, as we can see that ${\rm Var}_{\mathcal G}[d]$ is quite large for all KGs).
			This means that $\mathbb E_{L(\mathcal G)} [d] \gg \mathbb E_{\mathcal G} [d]$ in practice.
			On the other hand, the number of nodes in $L(\mathcal G)$ (which is $M$) is also far larger than the number of nodes in $\mathcal G$ (which is $N$).
			Therefore, $L(\mathcal G)$ is generally much larger and denser than its original graph $\mathcal G$.
			Based on Lemma \ref{lemma:1}, Theorem \ref{thm:2} is proven as follows:
 			\begin{proof}
 				In each iteration of relational message passing:
			
				The aggregation (Eq. (\ref{eq:relational_mp_1})) is performed for $M$ times, and each aggregation takes $\mathbb E_{L(\mathcal G)}[d] = \frac{N \cdot {\rm Var}_{\mathcal G}[d]}{M} + \frac{4M}{N} - 2$ elements as input in expectation.
				So the expected cost of aggregation in each iteration is $M \cdot \mathbb E_{L (\mathcal G)} [d] = N \cdot {\rm Var}_{\mathcal G}[d] + \frac{4M^2}{N} - 2M$;
			
				The update ((Eq. (\ref{eq:relational_mp_2}))) is performed for $M$ times, and each update takes 2 elements as input.
				Therefore, the cost of update in each iteration is $2M$.
			
				In conclusion, the expected cost of relational message passing in each iteration is $N \cdot {\rm Var}_{\mathcal G}[d] + \frac{4M^2}{N}$.
 			\end{proof}

 	\subsection{Proof of Theorem \ref{thm:3}}
	\label{sec:proof_3}
			\begin{proof}
				In each iteration of alternate relational message passing:
			
				The edge-to-node aggregation operation (Eq. (\ref{eq:alternate_mp_1})) is performed for $N$ times, and each aggregation takes $\mathbb E[d] = \frac{2M}{N}$ elements as input in expectation.
			Therefore, the expected cost of edge-to-node aggregation in each iteration is $N \cdot \mathbb E[d] = 2M$;
			
				The node-to-edge aggregation (Eq. (\ref{eq:alternate_mp_2})) is performed for $M$ times, and each aggregation takes 2 elements as input.
			So the cost of node-to-edge aggregation in each iteration is $2M$;
			
				The update (Eq. (\ref{eq:alternate_mp_3})) is performed for $M$ times, and each update takes 2 elements as input.
				Therefore, the cost of update in each iteration is $2M$.
			
				In conclusion, the expected cost of alternate relational message passing in each iteration is $6M$.
			\end{proof}

	\subsection{Implementation Details}
	\label{sec:implementation}
		\textbf{Baselines}.
		The implementation code of TransE, DistMult, ComplEx, and RotatE comes from \url{https://github.com/DeepGraphLearning/KnowledgeGraphEmbedding};
		the implementation code of SimplE is at \url{https://github.com/baharefatemi/SimplE};
		the implementation code of QuatE is at \url{https://github.com/cheungdaven/QuatE}, and we use QuatE$^2$ (QuatE without type constraints) here;
        the implementation code of DRUM is at \url{https://github.com/alisadeghian/DRUM}.
		For fair comparison, the embedding dimension for all the baselines are set to 400.
		We train each baseline for 1,000 epochs, and report the test result when the result on validation set is optimal.
		
		\xhdr{Our method}
		Our proposed method is implemented in TensorFlow and trained on single GPU.
    	We use Adam \cite{kingma2015adam} as the optimizer with learning rate of 0.005.
    	L2 regularization is used to prevent overfitting and the weight of L2 loss term is $10^{-7}$.
    	Batch size is 128, the number of epochs is 20, and the dimension of all hidden states is 64.
    	Initial relation features are set as their identities, while BOW/BERT features are studied in Section \ref{sec:model_variants}.
    	The above settings are determined by optimizing the classification accuracy on the validation set of WN18RR, and kept unchanged for all datasets.
    	
    	During experiments we find that performance of different number of context hops and the maximum path length largely depends on datasets, so these hyper-parameters are tuned separately for each dataset.
    	We present their default settings in Table \ref{table:hp}, and search spaces of hyper-parameters as follows:
		\begin{itemize}
			\item Dimension of hidden states: $\{8, 16, 32, 64\}$;
			\item Weight of L2 loss term: $\{10^{-8}, 10^{-7}, 10^{-6}, 10^{-5}\}$;
			\item Learning rate: $\{0.001, 0.005, 0.01, 0.05, 0.1\}$;
			\item The number of context hops: $\{1, 2, 3, 4\}$;
			\item Maximum path length: $\{1, 2, 3, 4\}$.
		\end{itemize}
		
		\begin{table}[t]
			\centering
			\small
			\setlength{\tabcolsep}{2pt}
			\begin{tabular}{c|cccccc}
				\hline
				& FB15K & FB15K-237 & WN18 & WN18RR & NELL995 & DDB14 \\
				\hline
				\#context hops & 2 & 2 & 3 & 3 & 2 & 3 \\
				Max. path len. & 2 & 3 & 3 & 4 & 3 & 4 \\
				\hline
			\end{tabular}
			\vspace{0.05in}
			\caption{Dataset-specific hyper-parameter settings: the number of context hops and the maximum path length.}
			\label{table:hp}
			\vspace{-0.2in}
		\end{table}
    	
    	Each experiment of \alg is repeated for three times.
    	We report average performance and standard deviation as the results.

	\subsection{More Results of Explainability on DDB14}
	\label{sec:ddb14_explain}
		After training on DDB14, we print out the transformation matrix of the context aggregator and the path aggregator in \alg, and the results are shown as heat maps in Figures \ref{fig:heatmap_neighbor} and \ref{fig:heatmap_path}, respectively.
		The degree of darkness of an entry in Figure \ref{fig:heatmap_neighbor} (Figure \ref{fig:heatmap_path}) denotes the strength of correlation between the existence of a contextual relation (a relational path) and a predicted relation.
		Relation IDs as well as their meanings are listed as follows for readers' reference:

		\begin{figure}[t]
    		\centering 
    		\includegraphics[width=0.4\textwidth]{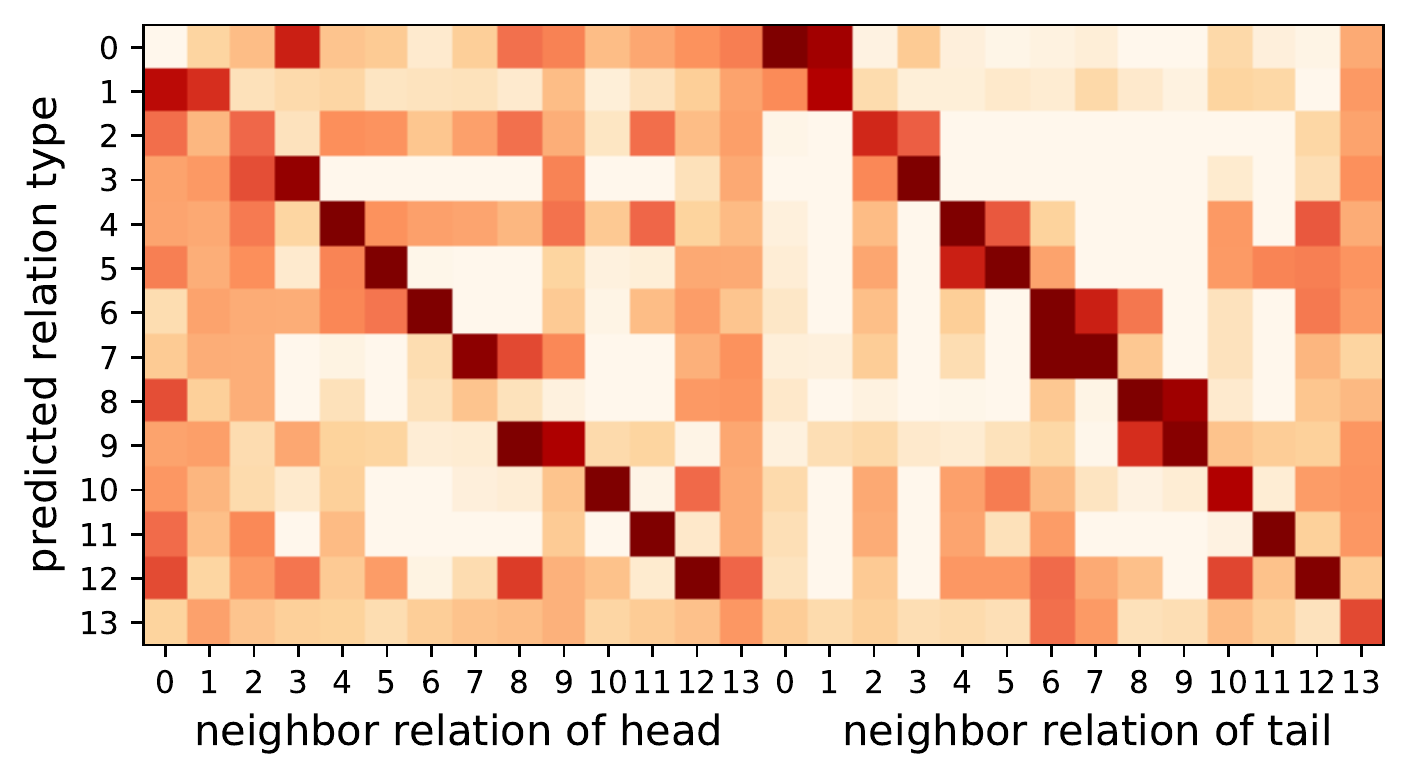}
    		\vspace{-0.1in}
    		\caption{The learned correlation between the contextual relations of head/tail and the predicted relations on DDB14.} 
    		\label{fig:heatmap_neighbor}
		\end{figure}

		\begin{table}[h]
			\small
			\centering
			\setlength{\tabcolsep}{8pt}
			\begin{tabular}{l|l}
				\hline
				0: belong(s) to the category of & 7: interacts with\\
			    1: is a category subset of & 8: belongs to the drug family of\\
			    2: may cause & 9: belongs to drug super-family\\
			    3: is a subtype of & 10: is a vector for\\
			    4: is a risk factor for & 11: may be allelic with\\
			    5: is associated with & 12: see also\\
			    6: may contraindicate & 13: is an ingredient of\\
			    \hline
			\end{tabular}
			\label{table:index}
			\vspace{-0.1in}
		\end{table}

		Figure \ref{fig:heatmap_neighbor} shows that most of large values are distributed along the diagonal.
		This is in accordance with our intuition, for example, if we want to predict the relation for pair $(h, ?, t)$ and we observe that $h$ appears in another triplet $(h, \texttt{is a risk factor for}, t')$, then we know that the type of $h$ is risk factor and it is likely to be a risk factor of other entities in the KG.
		Therefore, ``$?$'' are more likely to be ``\texttt{is a risk factor for}'' than ``\texttt{belongs to the drug family of}'' since $h$ is not a drug.
		In addition, we also find some large values that are not in the diagonal, e.g., (\texttt{belongs to the drug family of}, \texttt{belongs to the drug super-family}) and (\texttt{may contraindicate}, \texttt{interacts with}).
		
		We also have some interesting findings from Figure \ref{fig:heatmap_path}.
		First, we find that many rules from Figure \ref{fig:heatmap_path} is with the form:
		\begin{equation*}
			(a, \texttt{see also}, b) \wedge (b, \texttt{R}, c) \Rightarrow (a, \texttt{R}, c),
		\end{equation*}
		where \texttt{R} is a relation type in the KG.
		These rules are indeed meaningful because $(a, \texttt{see also}, b)$ means $a$ and $b$ are equivalent thus can interchange with each other.
		
		We also find \alg learns rules that show the relation type is transitive, for example:
		\begin{equation*}
		\begin{split}
			&(a, \texttt{is associated with}, b) \wedge (b, \texttt{is associated with}, c)\\[-0.05in]
			\Rightarrow &(a, \texttt{is associated with}, c);
		\end{split}
		\end{equation*}
		\begin{equation*}
		\begin{split}
			&(a, \texttt{may be allelic with}, b) \wedge (b, \texttt{may be allelic with}, c)\\[-0.05in]
			\Rightarrow &(a, \texttt{may be allelic with}, c).
		\end{split}
		\end{equation*}
		
		Other interesting rules learned by \alg include:
		\begin{equation*}
			(a, \texttt{belong(s) to the category of}, b) \Rightarrow (a, \texttt{is a subtype of}, b);
		\end{equation*}
		\begin{equation*}
			(a, \texttt{is a risk factor for}, b) \Rightarrow (a, \texttt{may cause}, b);
		\end{equation*}
		\begin{equation*}
			(a, \texttt{may cause}, c) \wedge (b, \texttt{may cause}, c) \Rightarrow (a, \texttt{may be allelic with}, b);
		\end{equation*}
		\vspace{-0.1in}
		\begin{equation*}
	    \begin{split}
			&(a, \texttt{is a risk factor for}, c) \wedge (b, \texttt{is a risk factor for}, c)\\[-0.05in]
			\Rightarrow &(a, \texttt{may be allelic with}, b).
		\end{split}
		\end{equation*}

\end{document}